\RequirePackage{xcolor}
\documentclass{article}
\usepackage{microtype}
\usepackage{graphicx}
\usepackage{booktabs} % for professional tables

\usepackage{hyperref}

\usepackage[accepted]{icml2020}

\icmltitlerunning{Learning Fair Policies in (Deep) Reinforcement Learning}

\usepackage[utf8]{inputenc}
\usepackage{url}
\usepackage{amsfonts}

\usepackage{tikz}
\usepackage[inline,shortlabels]{enumitem}
\usepackage{natbib,bm,amssymb}
\setitemize{wide,nosep,topsep=0pt,labelindent=10pt}
\setenumerate{wide,nosep,topsep=0pt,labelindent=10pt}
\setdescription{wide,nosep,topsep=0pt,labelindent=10pt}

\newcommand{\St}{{\mathcal S}}
\newcommand{\Ac}{{\mathcal A}}
\newcommand{\T}{{\bm P}}
\newcommand{\R}{{\bm r}}
\newcommand{\sd}{\bm d} % stationary distribution
\newcommand{\aR}{\mu} % average reward, which is a scalar
\newcommand{\gain}{\bm g} % gain, which is a vector
\newcommand{\vaR}{\bm \mu} % average vector reward, which is a vector of dim D
\newcommand{\vgain}{\bm G} % vector gain, which is a matrix S x D
\newcommand{\vf}{{\bm v}} % value function, which is a vector
\newcommand{\vvf}{{\bm V}} % vector value function, which is a matrix S x D
\newcommand{\vQ}{{\bm Q}} % 
\newcommand{\vJ}{\bm J} % vectorial objective function
\newcommand{\w}{{\bm w}} % GGI weights
\newcommand{\nO}{D} % number of objectives

\newcommand{\vR}{{\bm R}} % reward vector D columns
\usepackage{amsmath}%argmax
\DeclareMathOperator*{\argmax}{argmax}
\DeclareMathOperator{\GGItext}{GGF}
\newcommand{\GGI}{{\GGItext}_{\w}}
\newcommand{\GGF}{{\GGItext}_{\w}}

\usepackage{amsthm}%proof
\usepackage{amsmath}%align
\newtheorem{theorem}{Theorem}[section]
\newtheorem{lemma}[theorem]{Lemma}
\newtheorem{corollary}[theorem]{Corollary}
\newtheorem{example}[theorem]{Example}

\newcommand{\Expect}{\mathbb E}

\newcommand{\us}[1]{#1}
\newcommand{\pw}[1]{#1}
\newcommand{\mz}[1]{#1}

\usepackage{cancel}
\usepackage{xifthen}

\begin{document}
%%ICML template
\twocolumn[
\icmltitle{Learning Fair Policies in Multiobjective (Deep) Reinforcement Learning \\with Average and Discounted Rewards}



\icmlsetsymbol{equal}{*}

\begin{icmlauthorlist}
\icmlauthor{Umer Siddique}{ji}
\icmlauthor{Paul Weng}{ji,da}
\icmlauthor{Matthieu Zimmer}{ji}
\end{icmlauthorlist}

\icmlaffiliation{ji}{UM-SJTU Joint Institute, Shanghai Jiao Tong University, Shanghai, China}
\icmlaffiliation{da}{Department of Automation, Shanghai Jiao Tong University, Shanghai, China}

\icmlcorrespondingauthor{Paul Weng}{paul.weng@sjtu.edu.cn}

\icmlkeywords{Machine Learning, ICML}

\vskip 0.3in
]



\printAffiliationsAndNotice{}  % leave blank if no need to mention equal contribution

\begin{abstract}
As the operations of autonomous systems generally affect simultaneously several users, it is crucial that their designs account for fairness considerations.
In contrast to standard (deep) reinforcement learning (RL), we investigate the problem of learning a policy that treats its users equitably.
In this paper, we formulate this novel RL problem, in which an objective function, which encodes a notion of fairness that we formally define, is optimized.
For this problem, we provide a theoretical discussion where we examine the case of discounted rewards and that of average rewards.
During this analysis, we notably derive a new result in the standard RL setting, which is of independent interest:
it states a novel bound on the approximation error with respect to the optimal average reward of that of a policy optimal for the discounted reward.
Since learning with discounted rewards is generally easier, this discussion further justifies finding a fair policy for the average reward by learning a fair policy for the discounted reward.
Thus, we describe how several classic deep RL algorithms can be adapted to our fair optimization problem, and we validate our approach with extensive experiments in three different domains.
\end{abstract}
%%END ICML template

%%%%%%%%%%%%%%%%%%%%%%%%%%%%%%%%%%%%%%%%%%%%%
\section{Introduction}

The progress in artificial intelligence (AI) and its use in autonomous systems have created a lot of opportunities as well as challenges for human society. 
Indeed, a well-trained AI system can automate or solve some tasks better than humans \citep{PilarskiDawsonDegrisFahimiCareySutton11,SilverSchrittwieserSimonyanAntonoglouHuangGuezHubertBakerLaiBoltonChenLillicrapHuiSifreDriesscheGraepelHassabis17}. 
However, current AI methods do not handle well situations where they impact many users.

The usual approach in those AI systems consists in maximizing a single overall utility (measuring for instance efficiency, accuracy, or task fulfillment).
When an AI system affects many users, a \textit{utilitarian} objective is generally adopted, where the individual utilities of all users are summed (or equivalently averaged).
Yet, such approach does not guarantee that users are treated equitably.
Indeed, in order to obtain an efficient global solution for the utilitarian objective, the utility of some users may be sacrificed.
Thus, fairness considerations during the design of autonomous systems are critical if we want users to accept and trust them.

One typical solution to the previous issue is to resort to an \textit{egalitarian} approach where the objective becomes to maximize the utility of the worse-off user.
However, a direct application of this maxmin approach may not yield strictly efficient solution for all users because of the focus on only one user.
In this paper, we adopt a more refined definition of fairness \citep{Moulin04} that relies on three properties: efficiency, impartiality, and equity (Section~\ref{sec:ggi}).
In order to encode them, we use the generalized Gini evaluation function \citep{Weymark81} as the social welfare function (i.e., the function that defines the overall utility from all the user utilities).

In this work, we study the optimization of this fair welfare function in the context of (deep) reinforcement learning considering both discounted rewards and average rewards.

Our contributions can be summarized as follows:
\begin{enumerate}[wide]
\item We introduce this novel problem that we call ``fair optimization in RL'' (Section~\ref{sec:fairpolicies}).
\item We investigate its theoretical properties (Section~\ref{sec:theoreticaldiscussions}).
Notably, (1) we establish the sufficiency of stationary Markov policies for finding fair solutions,
(2) we discuss the possible state-dependency of fair optimality, and
(3) we provide an approximation error bound for using a policy optimal for discounted rewards instead of one optimal for average rewards. 
Interestingly, this last result applied to single-objective RL leads to a novel, simple, and  interpretable bound, which is of independent interest.
\item We adapt three deep RL algorithms for solving our fair optimization problem (Section~\ref{sec:algos}).
\item We provide extensive experimental results in three different domains (Section~\ref{sec:expes}), which validate our propositions.
\end{enumerate}
In the next section, we present the necessary background before presenting our contributions.

%%%%%%%%%%%%%%%%%%%%%%%%%%%%%%%%%%%%%%%%%%%%%
\section{Background}\label{sec:background}

In this section, we first recall Markov decision processes, and their extension to the multiobjective setting.
Then, we motivate and review the welfare function called generalized Gini social welfare function, which  encodes fairness.

\textbf{Notations.} Matrices are  denoted in uppercase and vectors in lowercase.
Both are written in bold.
The identity matrix is denoted $\bm I$.
Vectors are column vectors, except for those denoting probability distributions, which are row vectors.

\subsection{Markov Decision Processes}

A Markov Decision Process (MDP) \cite{Puterman94} is defined as a tuple of the following elements: 
a finite set of states $\St$, 
a finite set of actions $\Ac$, 
transition matrices $\T_a$ for each $a \in \Ac$ where $\T_{a,ss'}$ denotes the probability of reaching state $s'$ after performing action $a$ in state $s$, 
reward vectors $\R_a$ for each $a$ where $\R_{a,s}$ is the reward obtained after performing $a$ in $s$, and probability distribution $\sd_0$ over initial states.
In this model, a \textit{policy} $\pi$ defines a procedure that specifies how actions are selected in states.
A policy is \textit{stationary} if the same procedure is used at every time steps.
It is \textit{Markov} if it selects actions only based on the current state.
In this paper, unless otherwise stated, policies are stationary and Markov. 
A policy can be \textit{deterministic} (i.e., $\forall s, \pi(s) \in \Ac$) or \textit{stochastic} (i.e., $\forall s, a,  \pi(a \mid s)$ denotes the probability of selecting $a$ in $s$).
Deterministic policies are special cases of stochastic ones.
By extension, we write $\T_{\pi}$ for $\T_{\pi,ss'} = \sum_a \pi(a \mid s) \T_{a,ss'}$ and $\R_\pi$ for $\R_{\pi,s} = \sum_a \pi(a \mid s) \R_{a,s}$.
A policy induces a \textit{Markov reward process} whose transitions and rewards are resp. $\T_\pi$ and $\R_\pi$.

In reinforcement learning (RL),  $\T_a$'s and $\R_a$'s are usually unknown. 
The goal in an MDP or in RL is to find a policy that optimizes some performance measure, such as the \textit{expected discounted total reward} or \textit{expected average reward}. 

Using the discounted-reward criterion,
the value function $\vf_\pi$ of a policy $\pi$ from an initial state $s$ is defined by:
\begin{align}\label{eq:vf discounted_orig}
\vf_{\pi,s} = \Expect_{\T_\pi}\left[\sum_{t=1}^{\infty} \gamma^{t-1} \R_t \mid s\right],
\end{align}
where $\Expect_{\T_\pi}$ is the expectation taken with respect to $\T_\pi$, 
$\gamma \in [0, 1)$ is a discount factor, and $\R_t$ is the random variable that represents the reward obtained at time step $t$.
Given initial distribution $\sd_0$, an optimal policy is given by:
\begin{align}\label{eq:opt discounted}
    \argmax_\pi \sd_0 \vf_\pi .
\end{align}
Interestingly, the objective function can be rewritten as follows:
$\sd_0 \vf_\pi = \sd_{\pi_\gamma} \R_\pi$ where $\sd_{\pi_\gamma}$ is the \textit{discounted occupation distribution}\footnote{Technically speaking, it is not a probability distribution as it is not normalized.} over states of $\pi$, which is defined as $\sd_{\pi_{\gamma}} = \sum_{t=0}^\infty \gamma^t \sd_0 \T_\pi^t$ with $\T_\pi^0 = \bm I$ and $\T_\pi^t = \T_\pi^{t-1} {\T_\pi}$.
Value $\sd_{\pi_{\gamma}, s}$ represents the total discounted probability of visiting state $s$ under policy $\pi$ from initial distribution $\sd_0$.
A policy that is a solution of Problem~\eqref{eq:opt discounted} is called \textit{$\gamma$-optimal} and denoted $\pi^*_\gamma$.

Using the average-reward criterion, the value function of a policy $\pi$ is usually called  \textit{gain} and denoted $\gain_{\pi}$.
For an initial state $s$, it is defined by:
\begin{align} \label{eq:gain_orig}
\gain_{\pi,s} = \lim_{h \to \infty} \frac{1}{h} \Expect_{\T_\pi}\left[\sum_{t=1}^{h} \R_t \mid s\right].
\end{align}
Given a distribution over initial states $\sd_0$, the expected average reward $\aR_{\pi}$ obtained by a policy $\pi$ is defined by $\aR_{\pi} = \sd_0 \gain_{\pi}$.
It can also be expressed as $\aR_{\pi} = \sd_{\pi} \R_\pi$ where $\sd_{\pi}$ is the \textit{stationary distribution} of policy $\pi$, which is defined as the Ces\`aro-limit\footnote{
The Ces\`aro-limit of sequence $u_n$ as $n \to \infty$ is given by $\lim_{n \to \infty} \frac{1}{n} \sum_{i=0}^{n-1} u_i$.
It is a generalized notion of limit and is equal to the standard limit, if the latter exists.} 
of $\sd_0 \T_\pi^n$. 
Distribution $\sd_{\pi}$ represents the proportion of time policy $\pi$ spends in each state.
For the average-reward criterion, an optimal policy is obtained by:
\begin{align}\label{eq:opt average}
    \argmax_\pi \aR_{\pi}.
\end{align}
A policy that is a solution to this problem is called \textit{average-optimal} and denoted $\pi^*_1$.

The average reward criterion is often preferred in problems where the interaction between agent and environment goes on for a long time horizon. 
However, the two criteria are intimately connected \cite{BaxterBartlett01}, i.e., for any policy $\pi$, we have
$\sd_{\pi} \vf_{\pi} = \frac{\aR_{\pi}}{1 - \gamma}$.

In this paper, we assume that MDPs are \textit{weakly communicating}\footnote{
An MDP is \textit{weakly communicating} if its states can be partitioned into two classes:
one in which all states are transient under every stationary policy, and the other in which any two states can be reached from each other under some stationary policy.
}.
Interestingly, in such MDPs, the optimal gain $\gain_{\pi^*_1}$ is constant, i.e., independent of the initial state.

When the state or action space becomes too large or continuous, function approximation is needed to allow generalization.
With parametric function approximation (e.g., neural networks or linear function), 
a function $f$ is approximated by $\hat{f}(\bm\theta)$ where $\bm\theta$ denotes the parameters to be learned. 
In RL, both value functions or policies can be approximated.

Standard deep RL methods are usually designed for discounted rewards.
For instance, Deep Q Network (DQN) is an efficient extension of Q-Learning \cite{MnihKavukcuogluSilverRusuVenessBellemareGravesRiedmillerFidjelandOstrovskiPetersenBeattieSadikAntonoglouKingKumaranWierstraLeggHassabis15}.
DQN combines bootstrapping, off-policy updates and function approximation. 
To improve the learning stability it relies on experience replay \cite{Lin91} and target networks.
Two approximations of the Q value function $Q_\pi(s,a) = \Expect_{\T_ \pi}\left[\sum_{t=1}^{\infty} \gamma^{t-1} \R_t \mid s, a\right] $ are learned respectively parametrized by $\bm\theta$ and $\bm{\theta'}$.
The target network associated with $\bm{\theta'}$ is periodically updated towards $\bm{\theta}$. 
To update $\bm{\theta}$, the regression target is: 
$\hat{Q}_{\bm{\theta}}(s,a) = r + \gamma \max_{a' \in \Ac}\hat{Q}_{\bm{\theta'}}(s',a') $
where $(s,a,s',r)$ is a tuple drawn from the replay buffer respectively composed of a state, an action, a next state and a reward.

A policy can also be approximated and parametrized by $\bm\theta$. 
The policy gradient \cite{SuttonMcAllesterSinghMansour00} gives the direction in which the parameters should be updated: 
$\nabla_{\bm\theta} J(\pi_{\bm\theta}) \!=\! \Expect_{s \sim \sd_{\pi}, a \sim \pi_{\bm\theta}(\cdot \mid s)} [ A_\pi (s,a) \nabla_{\bm\theta} \text{log } \pi_{\bm\theta}(a | s) ] $ where $A_\pi (s,a) = Q_\pi (s,a) - V_\pi(s)$ is the advantage function. 
Since this function is unknown, it needs to be estimated, which can be done in different manners. 
In Advantage Actor-Critic (A2C), the advantage is estimated by $\hat{A}_{\text{A2C}}(s_t,a_t) = \sum_{t=1} \gamma^{t-1} r_t - \hat{V}(s_t)$ where $\hat{V}(s_t)$ is approximated by a critic network \cite{MnihBadiaMirzaGravesLillicrapHarleySilverKavukcuoglu16}. 
The A2C actor update derives from the policy gradient obtained from $J_\text{A2C}(\pi_{\bm\theta}) = \Expect_{s \sim \sd_{\pi}, a \sim \pi_{\bm\theta}(\cdot|s)} [ \hat{A}_\text{A2C} (s,a) ]$. 
In Proximal Policy Optimization (PPO), the advantage $\hat{A}_\text{PPO}$ is estimated with $\lambda$-returns \cite{SchulmanWolskiDhariwalRadfordKlimov17}. 
It also derives from the policy gradient but with an additional constraint mitigating the policy changes.
It is obtained from
 $J_\text{PPO}(\pi_{\bm\theta}) = \Expect_{s \sim \sd_{\pi}, a \sim \pi_{\bm\theta}(\cdot|s)} [ \text{min}(\rho_{\bm\theta} \hat{A}_\text{PPO} (s,a), \bar{\rho}_{\bm\theta} \hat{A}_\text{PPO} (s,a) ) ]$ where $\bar{\rho_{\bm\theta}} = \text{clip}(\rho_{\bm\theta}, 1 - \delta, 1 + \delta)$, $\rho_\theta = \frac{\pi_{\bm\theta}(a|s)}{\pi_b(a|s)}$, $\pi_b$ is the policy that generates the transitions and $\delta$ is a hyperparameter to control the constraint.

\subsection{Multiobjective Markov Decision Process}
A multiobjective MDP (MOMDP) is an MDP where rewards are vectors (instead of scalars) whose components, called \textit{objectives}, are interpreted as different criteria  (e.g., length, cost, duration) in the multicriteria setting, and as individual utilities in the multi-agent setting. 
Formally, the reward function of an MOMDP is redefined as follows: $\vR_{a,s} \in \mathbb R^\nO$ where $\nO$ is the number of objectives.
Consequently, value functions (now denoted $\vvf$, $\vQ$, $\vgain$) also take values in $\mathbb R^\nO$.

All the previous definitions for MDPs extend naturally to MOMDPs.
Notably, with discounted rewards, \eqref{eq:vf discounted_orig} becomes:
\begin{align}\label{eq:vf discounted}
\vvf_{\pi,s} = \Expect_{\T_\pi}\left[\sum_{t=1}^{\infty} \gamma^{t-1} \vR_t \mid s\right] ,
\end{align}
where $\vvf_\pi$ can be seen as a $|\St | \times \nO$ matrix and $\vR_t$ represents the random
vector reward obtained at time step $t$.

With average reward, the gain \eqref{eq:gain_orig} becomes:
\begin{align} \label{eq:gain}
\vgain_{\pi,s} = \lim_{h \to \infty} \frac{1}{h} \Expect_{\T_\pi}\left[\sum_{t=1}^{h} \vR_t \mid s\right],
\end{align}
where $\vgain_{\pi}$ can be seen as a $|\St| \times \nO$-dimensional matrix.

Multiobjective optimization in MOMDPs amounts to solving the following problem: $\argmax_\pi \vJ(\pi)$ where $\vJ(\pi)$ is the multiobjective version of either \eqref{eq:opt discounted} or \eqref{eq:opt average}, and the vector maximization is with respect to \textit{Pareto dominance}\footnote{ 
$\forall \bm v, \bm v' \in \mathbb R^\nO$, $\bm v$ \textit{weakly Pareto-dominates} $\bm v' \Leftrightarrow \forall i, \bm v_i \ge \bm v'_i$. 
Besides, $\bm v$ \textit{Pareto-dominates} $\bm v' \Leftrightarrow \forall i, \bm v_i \ge \bm v'_i$ and $\exists j, \bm v_j > \bm v'_j$.}.
As there is no risk of confusion, Pareto dominance is simply denoted $\ge$ for its weak form and $>$ for its strict form.

A policy whose value function (or gain) is not Pareto-dominated is called \textit{Pareto-optimal}.
The usual approach in MOMDPs is to compute all the Pareto-optimal solutions.
However, the number of such solutions may be very large in some problems and such approach may be infeasible in general.
Indeed, there are some MOMDP instances where the number of Pareto-optimal policies are exponential in the MOMDP size \citep{PernyWengGoldsmithHanna13UAI}.

In practice, and especially in RL with an autonomous agent, one is rather interested to focus on one solution, typically one that finds a good balance between all the objectives.
One naive way to focus on only one solution is to use a weighted sum to combine the objectives. 
However, this technique does not provide any control on how balanced the objective values are.
A better method is to use a non-linear function to combine the objectives.
In our context of multiple users, finding balanced solutions amounts to finding fair solutions.
We detail our approach for fairness next.

\subsection{Generalized Gini Social Welfare Function}\label{sec:ggi}
In this paper, we require an optimal solution to satisfy three properties to quality as a fair solution:
\begin{description}
\item[Efficiency] A fair solution should be Pareto-optimal.
\item[Impartiality] A fair solution should satisfy the \textit{``equal treatment of equals''} principle, which states that users with identical characteristics should be treated similarly.
\item[Equity] A fair solution should satisfy the \textit{Pigou-Dalton principle} \citep{Moulin04}.
Intuitively, this principle states that given a utility vector $\bm v \in \mathbb R^\nO$, a transfer from a better-off user to a worse-off user yields a new vector that should be preferred.
Formally, for any indices $i$ and $j$, if $\bm v_i > \bm v_j$, then for any $\epsilon$ such that $\bm v_i - \bm v_j >\epsilon>0$, the new vector $\bm v - \epsilon \bm e_i + \epsilon \bm e_j$ is preferred to $\bm v$, where $\bm e_i$ denotes the $i-$th canonical basis vector\footnote{Vector $\bm e_i$ is such that $\bm e_{ij} = 0$ for $i\neq j$ and $\bm e_{ii}=1$.}.
\end{description}

The first property is natural because choosing a Pareto-dominated solution would be irrational.
The second one is reasonable in the context of fairness.
It holds in our work by assumption: we assume that all the objectives are equal and should therefore be treated in the same way. 
The third one is the key property in the context of fair optimization, as it captures in a natural way the idea that we prefer solutions whose utility distribution over users is balanced.

In order to implement concretely those three principles, we resort to a welfare function called the \textit{generalized Gini social welfare function} (GGF) \citep{Weymark81}.
GGF is defined as follows: %(see Appendix A for a discussion):
\begin{equation}\label{eq:ggi}
    \GGI(\bm v) = \sum_{i=1}^\nO \w_{i} \bm v^\uparrow_{i},
\end{equation}
where $\bm v \in \mathbb R^\nO$, $\w \in \mathbb R^\nO$ is a fixed positive weight vector whose components are strictly decreasing (i.e., $\w_1 > \ldots > \w_\nO$), and
$\bm v^\uparrow$ corresponds to the vector with the components of vector $\bm v$ sorted in an increasing order (i.e., $\bm v^\uparrow_1 \le \ldots \le \bm v^\uparrow_\nO$). 
Furthermore, we assume without loss of generality that the GGF weight vector $\w$ is normalized and sum to one (i.e., $\w \in [0, 1]^\nO$ and $\sum_{i=1}^{\nO} \w_{i} = 1$).

GGF satisfies the required three properties \citep{Weymark81}.
As the GGF weights are positive, GGF is monotonic with respect to Pareto dominance.
It therefore satisfies the efficiency property.
Because the components of $\bm v$ are reordered in \eqref{eq:ggi}, GGF is symmetric with respect to its components.
It therefore satisfies the impartiality property.
Finally, because the GGF weights are positive and decreasing, GGF is Schur-concave (i.e., it is monotonic with respect to Pigou-Dalton transfers).
It therefore satisfies the equity property.

Besides, GGF is a piecewise-linear concave function.
Indeed, it is easy to check that GGF can be rewritten as follows thanks to its positive decreasing weights:
\begin{align}\label{eq:ggi min}
    \GGI(\bm v) = \min_{\sigma \in \mathbb S_\nO} \w_\sigma^\intercal \bm v, % = \min_{\sigma \in \mathbb S_\nO} \sum_{i=1}^\nO \w_\sigma(i) 
\end{align}
where $\mathbb S_\nO$ is the symmetric group of degree $\nO$ (i.e., set of permutations over $\{1, \ldots, \nO\}$), 
$\sigma$ is a permutation, and
$\w_\sigma = (\w_{\sigma(1)}, \ldots, \w_{\sigma(\nO)})$. 
Equation~\eqref{eq:ggi min} holds since the minimum is attained by assigning the largest weight to the smallest component of $\bm v$, the second-largest weight to the second-smallest component of $\bm v$, and so on.

\mz{
Although GGF is not the only fair welfare function, it enjoys nice properties:
(1) simplicity, as it is a weighted sum in the Lorenz space \citep{chakravartyEthicalSocialIndex1990,PernyWengGoldsmithHanna13UAI},
(2) its well-understood properties axiomatized by  \citet{Weymark81},
(3) its generality.

GGF can cover various special cases by setting its weights appropriately, e.g.:
\begin{itemize}
    \item If $w_1 \to 1, w_2 \to 0, ..., w_D \to 0$, GGF corresponds to the maxmin egalitarian notion of fairness \citep{Rawls71}.
    \item If $w_1 \to 1, w_2 \to \varepsilon, ..., w_D \to \varepsilon$, GGF corresponds to the regularized maxmin egalitarian notion of fairness.
    \item If $w_1 \to 1/D, ..., w_D \to 1/D$, GGF corresponds to the utilitarian approach.
    \item If $w_k/w_{k+1} \to +\infty$, GGF corresponds to the leximin notion of fairness \citep{Rawls71,kurokawaLeximinAllocationsReal2015}.
\end{itemize}
}

%%%%%%%%%%%%%%%%%%%%%%%%%%%%%%%%%%%%%%%%%%%%%
\section{Fair Policies in RL} \label{sec:fairpolicies}
By integrating GGF with MOMDPs, we can now formally formulate the \textit{fair optimization} problem investigated in this paper, which is the problem of determining a policy that generates a fair distribution of rewards to $D$ fixed users:
\begin{equation}\label{eq:ggi pb}
    \argmax_{\pi} \GGI (\vJ(\pi)) ,
\end{equation}
where $\vJ(\pi)$ can be defined with the discounted or average reward.
As GGF is a concave function, \eqref{eq:ggi pb} defines a convex optimization problem.
This problem defined with  discounted rewards is called \textit{GGF-$\gamma$ problem}, while that with average rewards is called \textit{GGF-average problem.}
Their solutions are respectively called \textit{GGF-$\gamma$-optimal} and \textit{GGF-average-optimal} policies.

In this paper, we aim at solving this problem in the RL setting. 
As GGF is a non-linear function, fair optimization is a non-linear convex optimization problem. 
This brings novel difficulties, which we discuss next. 

\subsection{Theoretical Discussions}\label{sec:theoreticaldiscussions}

In this part, we discuss three important points related to fair optimization in MOMDPs:
\begin{enumerate*}[(i),mode=unboxed, afterlabel={{\nobreakspace}}]
    \item which subset of policies is guaranteed to contain an optimal solution,
    \item fair solution may depend on initial states, and
    \item how close is the GGF of the average vector reward of a GGF-$\gamma$-optimal policy to that of the optimal average vector reward.
\end{enumerate*}
The proofs of our theoretical results can be found in Appendix A.

\paragraph{Sufficiency of Stationary Markov Policies.}
A first question related to Problem~\eqref{eq:ggi pb} is which types of policies are optimal among the set of all (possibly non-stationary) policies.
The following lemma, which has not been stated and proved formally to the best of our knowledge, shows that there always exists a GGF-($\gamma$ or average)-optimal stationary stochastic Markov policy for Problem~\eqref{eq:ggi pb}.

\begin{lemma}\label{lem:ssm}
For any MOMDP, Problem~\eqref{eq:ggi pb} admits a solution that is a stationary stochastic Markov policy.
\end{lemma}

Note that this result holds in fact for any monotonic function, not only GGF in Problem~\eqref{eq:ggi pb}.
It implies that one can search for an optimal policy in the smaller set of stationary stochastic Markov policies instead of the set of all policies.
Also, note that contrary to the single-objective case, a deterministic policy may not be optimal \citep{Busa-FeketeSzorenyiWengMannor17} because fairer solution can be obtained via randomization.

\paragraph{Possibly State-Dependent Optimality.}
For the GGF-$\gamma$ problem, it is known that optimality depends on initial states or more generally on the distribution over initial states.
(see Example~\ref{ex:ggfgamma}). 
\mz{
\begin{example}\label{ex:ggfgamma}
\tikzstyle{vertex}=[draw,circle,minimum size=15pt,inner sep=0pt]
\tikzstyle{selected vertex} = [vertex, fill=red!24]
\tikzstyle{edge} = [draw,thick,->]
\tikzstyle{weight} = [font=\small]
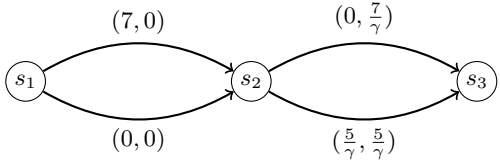
\begin{figure}[ht]
\centering
\begin{tikzpicture}[scale=1, auto,swap]
\footnotesize
    % First we draw the vertices
    \foreach \pos/\name/\label in {{(0,0)/x1/s_1}, {(3,0)/x2/s_2}, {(6,0)/x3/s_3}}
        \node[vertex] (\name) at \pos {$\label$};
    % Connect vertices with edges and draw weights
    \foreach \source/\dest/\weight in {x1/x2/{(7, 0)}, x2/x3/{(0, \frac{7}{\gamma})}}
        \path[edge] (\source) edge [bend left] node[weight,above] {$\weight$} (\dest);
    \foreach \source/\dest/\weight in {x1/x2/{(0, 0)}, x2/x3/{(\frac{5}{\gamma}, \frac{5}{\gamma})}}
        \path[edge] (\source) edge [bend right] node[weight,below] {$\weight$} (\dest);
\end{tikzpicture}
\caption{Example of MDP where optimality for Problem~\eqref{eq:ggi pb} with discounted rewards depends on states.}
\label{fig:ex}
\end{figure}
The following example adapted from \citep{OgryczakPernyWeng13} shows that a GGF-$\gamma$-optimal policy depends on initial states.
Consider the deterministic bi-objective three-state MDP depicted in Figure~\ref{fig:ex} where arcs represent actions, arc weights correspond to vector rewards, and  $\gamma \in [0, 1)$ is a discount factor.
Each state has two actions (Up, Down), except state $s_3$, which is an absorbing state.
Assume $\w = (5/9, 4/9)$.
Then, from $s_1$, the optimal policy $\pi_1$ chooses Up in $s_1$, and Up in $s_2$.
However, viewed from $s_2$, $\pi_1$ is not optimal, because action Down is preferred in $s_2$.
\end{example}
}

This point raises a potential difficulty when applying a $\gamma$-optimal policy: depending on which state is visited, one may have an incentive to switch to another policy.
In decision theory, this issue is called \textit{dynamic inconsistency of preferences} \citep{McClennen90}.
Besides, it implies that the Bellman principle of optimality does not hold anymore and therefore dynamic programming cannot be directly used for computing fair optimal solutions. 
However, we can prove that for with  average reward, preferences stay state-independent as the gain of the optimal policy is constant.

\begin{lemma} \label{lem:constant gain}
For any weakly-communicating MOMDP,  
the GGF-average problem
admits a solution that is a stationary stochastic Markov policy with constant gain.
\end{lemma}

\paragraph{Approximation Error}
A common practice in RL is to find an approximately optimal policy for the average reward by solving the related discounted reward problem.
In the single-objective case, \cite{Kakade01} proved that doing so, the difference between the gain of a $\gamma$-optimal policy and the optimal gain can be bounded:
\begin{theorem} \label{thm:Kakade}
Assume $\T_{\pi^*_1}$ has $n$ distinct eigenvalues.
Let $\bm U = (\bm u_1, \ldots, \bm u_{|\St|})$ be the matrix of its right eigenvectors with the corresponding eigenvalues $\lambda_1=1 > |\lambda_2| \ge \ldots \ge |\lambda_{|\St|}|$.
Then,
\vspace{-1em}
\begin{align*}
    \aR_{\pi^*_\gamma} \ge \aR_{\pi^*_1} - \kappa_2(\bm U) \|\R\| \frac{1 - \gamma}{1 - \gamma |\lambda_2|},
\end{align*}
where $\kappa_2(\bm U) = \|\bm U\|_2\|\bm U^{-1}\|_2$ is the condition number of $\bm U$, $\|\bm U\|_2 = \max_{\bm x : \|\bm x\|=1} \|\bm U \bm x\|$, and $\|\bm x\|$ is the Euclidean norm for any vector $\bm x$.
\end{theorem}
To the best of our knowledge, this is the only known bound for such approximation error. % on the average reward.
However, it only holds if $\T_{\pi^*_1}$ has $n$ distinct eigenvalues and it may be hard to interpret and evaluate in practice as it involves the condition number of the corresponding eigenvector matrix.

This motivates us to prove an alternative bound.
Using a matrix decomposition approach based on Laurent series expansion and Drazin generalized inverse, \citet{LamondPuterman89} proved the following relation between the discounted value function and the gain of a stationary policy.
\begin{theorem} \label{thm:vf decomposition}
For any MDP, any stationary policy $\pi$, and any $\gamma \in (\frac{\sigma(\bm H_{\T_\pi})}{\sigma(\bm H_{\T_\pi})+1}, 1)$,
\vspace{-.5em}
\begin{align}\label{eq:decomposition}
    \vf_\pi = \frac{1}{1 - \gamma} \gain_{\pi} + \frac{1}{\gamma} \sum_{n=0}^{\infty} \left(\frac{\gamma-1}{\gamma}\right)^n \bm{H}_{\T_\pi}^{n+1} \R_\pi,
\end{align}
where 
$\bm{H}_{\T_\pi}$ is the Drazin inverse of $\bm I - \T_\pi$, which is given by $(\bm I - \T_\pi + \T_\pi^*)^{-1}(\bm I - \T_\pi^*)$, %and 
$\T_\pi^*$ is the Ces\`aro-limit of $\T_\pi^n$ for $n \to \infty$, and
$\sigma(\bm H_{\T_\pi})$ is the spectral radius of matrix $\bm H_{\T_\pi}$.
\end{theorem}

Using Th.~\ref{thm:vf decomposition}, and assuming $\gamma$ close enough to $1$, we can prove an error bound for GGF (policies are GGF-optimal):
\begin{theorem}\label{thm:momdp bound}
For any weakly-communicating MOMDP:
\vspace{-.5em}
\begin{align*}
&\GGI(\vaR_{\pi^*_\gamma}) 
     \ge \GGI(\vaR_{\pi^*_1}) \\
     & \qquad - \overline\vR (1-\gamma) \left(
     \rho(\gamma, \sigma(\bm{H}_{\T_{\pi^*_1}})) + \rho(\gamma, \sigma(\bm{H}_{\T_{\pi^*_\gamma}})) \right),
\end{align*}
where $\overline\vR = \max_\pi \|\vR_\pi\|_1$ and
$\rho(\gamma, \sigma) = \frac{\sigma}{\gamma -(1-\gamma) \sigma}$.
\end{theorem}

Interestingly, Theorem~\ref{thm:momdp bound} applied to the single-objective case ($\nO=1$) yields an alternative approximation error bound, which is more general than that of Theorem~\ref{thm:Kakade}.
\begin{corollary}\label{cor:bound MDP}
For any weakly-communicating MDP:
\vspace{-.5em}
\[
\aR_{\pi^*_\gamma} 
     \ge \aR_{\pi^*_1} - \overline\R (1-\gamma) \left(
     \rho(\gamma, \sigma(\bm{H}_{\T_{\pi^*_1}})) + \rho(\gamma, \sigma(\bm{H}_{\T_{\pi^*_\gamma}})) \right)
\]
where $\overline\R = \max_\pi \|\R_\pi\|$. %, 
\end{corollary}

The bounds in Theorem~\ref{thm:momdp bound} and Corollary~\ref{cor:bound MDP} clearly show that when $\gamma \to 1$, the approximation error tends to zero as expected.
Like in the bound in Theorem~\ref{thm:Kakade}, the error depends on instance-specific constants.
Here, it mainly depends on the spectral radius of the Drazin inverse of $\bm I - \T_\pi$ where $\pi$ is either a GGF-$\gamma$-optimal or GGF-average-optimal policy.
This spectral radius can be intuitively interpreted as a measure of how long the policy could spend in transient states (and therefore, as a measure of how long the policy takes to converge to its average reward).
Indeed, a larger $\sigma(\bm{H}_{\T_{\pi}})$ implies a larger $\rho(\gamma, \sigma(\bm{H}_{\T_{\pi}}))$ and a larger bound.

%%%%%%%%%%%%%%%%%%%%%%%%%%%%%%%%%%%%%%%%%%%%%
\section{Algorithms}\label{sec:algos}
In this section, we explain how to modify the DQN and policy gradient algorithms in order to solve Problem~\eqref{eq:ggi pb} with discounted rewards.
As previously discussed, it can provide approximate solutions to that problem with average rewards for $\gamma$ close enough to $1$.

\paragraph{DQN.}

To optimize GGF, we modify the deep Q network (DQN) to
take their values in $\mathbb R^{|\Ac| \times \nO}$ instead of $\mathbb R^{|\Ac|}$.
DQN is trained to predict the multiobjective $\vQ$ function.
Note that directly predicting the GGF values would have prevented bootstrapping.
Thus, the regression target of DQN becomes:
\begin{align*}
\hat{\vQ}_{\theta}(s,a) &= \us{\R} + \gamma \hat{\vQ}_{\theta'}(s', a^*),
\end{align*}
where $a^* =  \argmax_{a' \in \Ac} \ \GGI\big(\us{\R} + \gamma  \hat{\vQ}_{\theta'}(s',a')\big)$.
This adapted version of DQN is called \textit{GGF-DQN}.

Note that ideally, $a^*$ should be selected as $\argmax_{a' \in \Ac} \ \GGI\big(\Expect_{s'}\left[\us{\R} + \gamma  \hat{\vQ}_{\theta'}(s',a')\big)\right]$.
However, it is hard to compute even if the expectation is estimated by a sample mean.
Therefore, our modification of DQN optimizes in fact an expectation of GGF, and not a GGF of an expectation.
By Jensen inequality, this implies that we are actually optimizing a lower bound of the correct objective function.

\paragraph{Policy Gradient Methods.}
A natural alternative approach for solving Problem~\eqref{eq:ggi pb} is to use a policy gradient method.
Contrary to the adaptation of DQN, it directly optimizes the desired objective function.
Another advantage is that it can learn a stochastic policy, which may strictly dominate a deterministic one for GGF.

The policy gradient is formulated as follows for GGF:
\begin{align*}
\nabla_{\bm\theta} \GGI (\vJ(\pi_{\bm\theta})) =& \nabla_{\vJ(\pi_{\bm\theta})} \GGI (\vJ(\pi_{\bm\theta})) \cdot \nabla_{\bm\theta} \vJ(\pi_{\bm\theta}) \\
=& \w_\sigma^\intercal \nabla_{\bm\theta} \vJ(\pi_{\bm\theta}),
\end{align*}
where $\nabla_{\bm\theta} \vJ(\pi_{\bm\theta})$ is a $\nO \times N$ matrix representing the classic policy gradient over the $\nO$ objectives, $\w_\sigma$ is sorted according to $\vJ(\pi_{\bm\theta})$ and $N$ is the number of policy parameters.

In the experiments, we applied it to PPO and A2C:
$\nabla_{\bm\theta} \vJ(\pi_{\bm\theta})$ is respectively replaced by $\nabla_{\bm\theta} \vJ_\text{PPO}(\pi_{\bm\theta})$ and $\nabla_\theta \vJ_\text{A2C}(\pi_{\bm\theta})$. 
In order to sort $\w$, the initial states are stored to empirically estimate $\vJ(\pi_{\bm\theta}) = \Expect_{s_0 \sim \sd_0} [\hat{\vvf}(s_0)]$ where $\hat{\vvf} : \St \to \mathbb R^\nO$ is approximated with the critic.
The resulting algorithms are called respectively \textit{GGF-PPO} and \textit{GGF-A2C}.

%%%%%%%%%%%%%%%%%%%%%%%%%%%%%%%%%%%%%%%%%%%%%
\section{Experimental Results}\label{sec:expes}
To test our three algorithms, we carried out experiments in three different domains (for detailed descriptions, see Appendix B):
\begin{enumerate*}[(i),mode=unboxed, 
afterlabel={{\nobreakspace}}, 
]
\item Species conservation (SC),
\item Traffic light control (TL),
\item Data center control (DC).
\end{enumerate*}
The first domain (SC) corresponds to a conservation problem encountered in ecology, where the goal is to maintain the populations of several interacting endangered species. 
We adapt the two-species model proposed by \citet{chades2012setting} to specifically take into account fairness with respect to the two species, namely an endangered species (sea otters) and its prey (northern abalone).
A state encodes the population numbers of the species.
The transition function is based on the population growth models for both species taking into account factors such as  poaching (for abalones) or oil spills (for sea otters).
In order to keep the two populations balanced, five actions are considered: do nothing
introduce sea otters, enforce antipoaching, 
control sea otters, and 
one-half antipoaching and one-half control sea otters.
Vector rewards correspond to scaled species  densities (in $m^{-2}$).

The second domain (TL) corresponds to the classic traffic light control problem, in which an agent controls the traffic lights at one intersection in order to optimize traffic flow.
The usual approach to this problem amounts to minimizing the expected sum of waiting times over all lanes.
Instead, we propose to take into account fairness with respect to each road ($D=4$ in our experiments).
In other words, the goal is to learn a controller that optimizes the expected waiting times per road.
More specifically, we consider an eight-lane intersection, where the four directions have 2 lanes.
A state is composed of the total waiting time and density of cars (in $[0, 1]$) stopped at the intersection in each lane.
It also contains the current phase (i.e., which lanes and directions have green/red lights) of the traffic lights.
An action corresponds to a change of phase.
Traffic randomly generated with fixed distributions.
At each time step vehicles are emitted randomly with the given probability by following a binomial distribution. The binomially distributed flow approximates a Poisson distribution for small probabilities that a number of random events happens with certain rate independently. 
A reward is a vector whose components are the sums of negative waiting times for each lane.

The third domain (DC) is a data center control problem, where a centralized controller manages a  computer network that is shared by a certain number of hosts in order to optimize the bandwidths of each host \citep{iroko}. 
Here, fairness is expressed with respect to hosts ($D=16$ in our experiments).
This kind of problem can typically occur in software-defined networking (SDN) for instance.
A state encodes network statistics (i.e., queue length, derivative over time of queue length, number of packet drops, and queue length above some limit) and the current bandwidth allocation to hosts.
An action is a vector of bandwidth allocation.
Traffic between hosts is randomly generated.
A reward is a vector whose components are  bandwidths per host penalized by a sum of queue lengths (in order to avoid bufferbloat).

The three domains are roughly ranked in increasing complexity (with also increasing number of objectives).
The first two domains have discrete state-action spaces, while the third has continuous state-action spaces.
As they are all episodic problems, they are all communicating (MO)MDPs.

On those domains, we typically ran DQN, PPO, A2C, and their adaptations to GGF (i.e., GGF-DQN, GGF-PPO, GGF-A2C).
The hyperparameters of the algorithms were optimized (Appendix C) and all the experiments were conducted using Lightweight HyperParameter Optimizer (LHPO), an open source library used to run parallel experiments on a cluster \cite{zimmer2018phd}.
Two computers with double CPU sockets have been used (Intel Xeon CPU E5-2678 v3).
The unnormalized GGF coefficients are defined as $\w_i = \frac{1}{2^i}$ from 0 to $\nO-1$.
All the experimental results (e.g., plots) are averaged over 50 (resp. 20) runs with different seeds for SC and TL (resp. DC as it is a complex environment).

We now present the \pw{main} results of ours experiments (for more, see Appendix D).
They have been designed to answer the following questions:
\begin{enumerate*}[(A),mode=unboxed, afterlabel={{\nobreakspace}},font={\bfseries}]
\item What is the impact of optimizing GGF instead of the average of the objectives?
\item How do the algorithms adapted to GGF compare with each other and with their standard versions? 
\item How do fair deterministic and stochastic policies compare?
\item What is the effect of $\gamma$ with respect to GGF-average optimality?
\item How do those algorithms perform in continuous domains?
\end{enumerate*}

\paragraph{Question (A)}

In order to answer (A), we discuss the experimental results of \us{DQN, A2C,} PPO and their GGF \us{counterparts} in the SC domain.
We can first compare them in terms of the average over the two accumulated densities \pw{during learning phase (see
Figure~\ref{fig:species-ppo-accumulative}).}
As a sanity check, \pw{the figure} also includes the uniformly random policy.
As densities are accumulated over an episode, drops happen in the curves of \pw{all} algorithms.
As expected, the random policy performs the worse and \us{standard RL algorithms} 
are the best because the average density is roughly what is optimized \us{in those algorithms.}
An algorithm adapted for GGF would normally perform worse than its original version in terms of the average of the objectives, since it trade-offs between efficiency and equity.
\begin{figure}[tb]
\begin{center}
\centering
\includegraphics[width=1.0\linewidth]{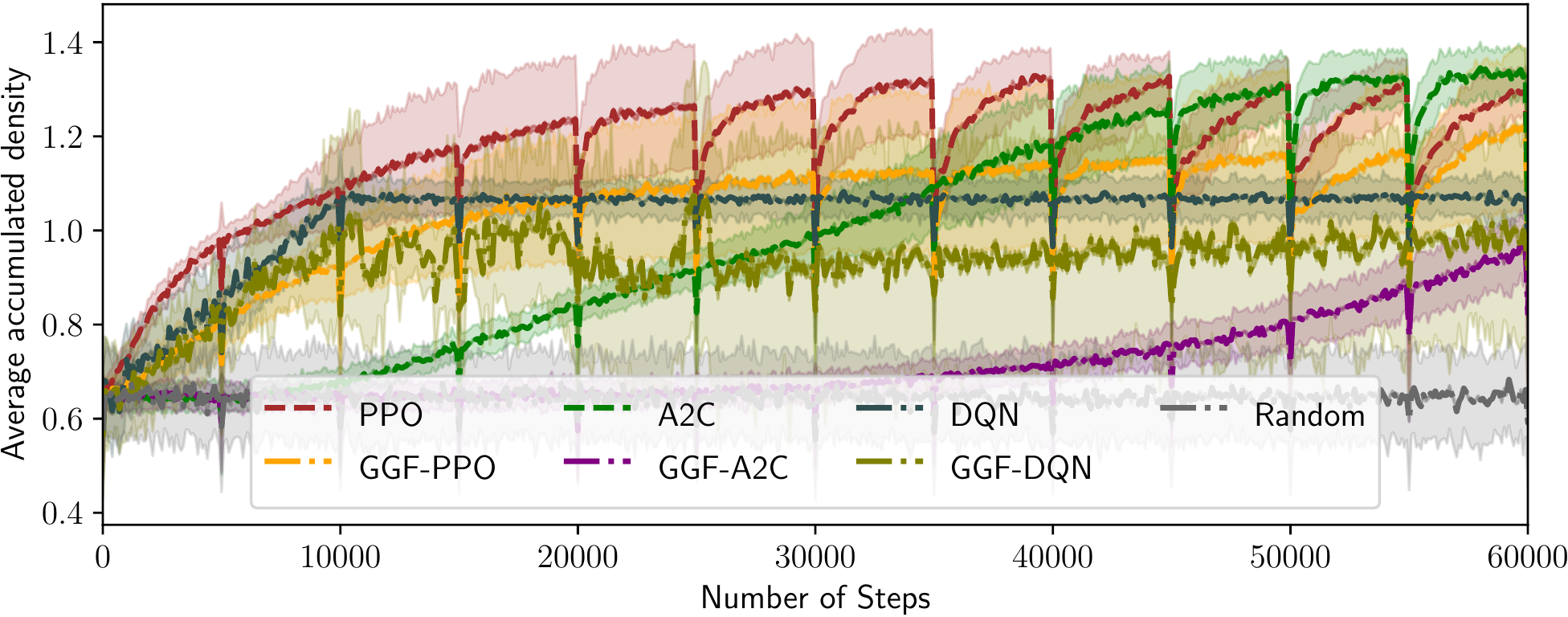}
\caption{Average accumulated densities of \us{DQN, A2C, PPO and their GGF versions} during the learning phase with those of the random policy in the SC domain.} 
\label{fig:species-ppo-accumulative}
\end{center}
\end{figure}

More interestingly, for our purposes, we can compare the algorithms in terms of the GGF score.
After training, the obtained policies are applied 50 times in the environment.
This score is the GGF of the sample average vector rewards of the generated trajectories.
Figure~\ref{fig:species-ppo-box} shows the distribution of this score for the policies learned by \us{DQN, A2C, PPO and their GGF algorithms.}
The number of steps during training and testing is 3650 (corresponding to 10 years).
As expected, \us{all the three GGF algorithms have higher GGF score than their original algorithms.} 
\begin{figure}[tb]
\vskip 0.2in
\begin{center}
\includegraphics[width=1.0\linewidth]{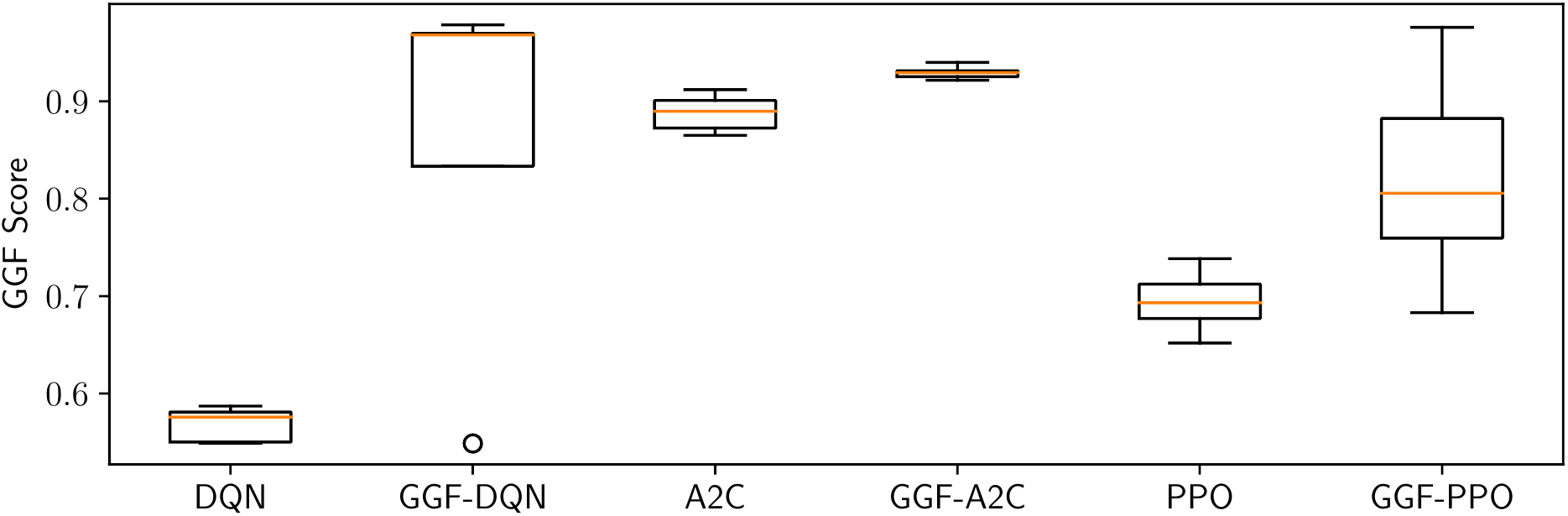}
\caption{GGF scores of \us{DQN, A2C, PPO and their GGF algorithms} during the testing phase in the SC domain.} 
\label{fig:species-ppo-box}
\end{center}
\vskip -0.2in
\end{figure}

As the GGF score does not directly give a clear picture of how balanced the objectives are, we also provide the plots of non-aggregated accumulated densities estimated after training (Figure~\ref{fig:species-separated-densities}), which can easily be done for the SC domain as it is bi-objective.
We can observe again that standard RL algorithms obtain higher total accumulated densities than their GGF counterparts.
However, the individual densities of the two species for the standard approaches are much more unequal than those obtained with our approach, which provides much fairer solutions.

Because we cannot easily display the non-aggregated objectives in all the domains, we introduce additional statistics to evaluate fairness.
Notably, the \textit{Coefficient of Variation} (CV), which can be understood as a simple measure of inequality.
In Figure~\ref{fig:species-cv}, every algorithm optimizing GGF have a lower CV and a higher minimum density.

\begin{figure}[tb]
\begin{center}
\centering
\includegraphics[width=1.0\linewidth]{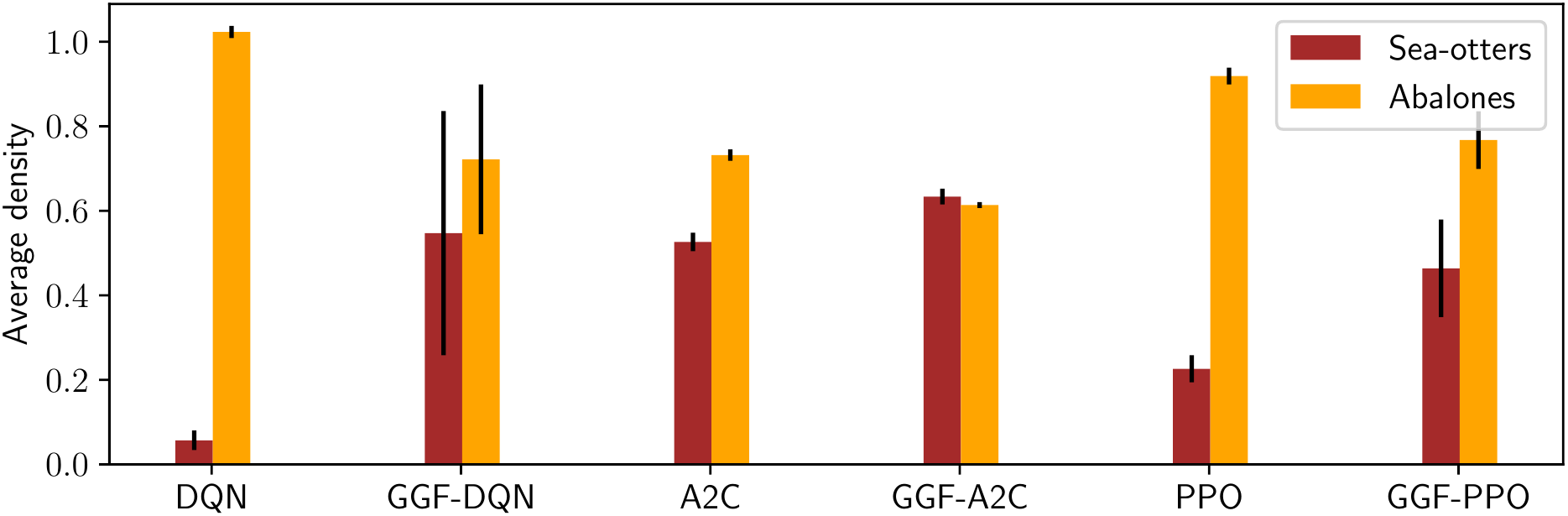}
\caption{Individual densities for DQN, A2C, PPO and their GGF versions during the testing phase in the SC domain.} 
\label{fig:species-separated-densities}
\end{center}
\end{figure}

\begin{figure}[tb]
\begin{center}
\centering
\includegraphics[width=1.0\linewidth]{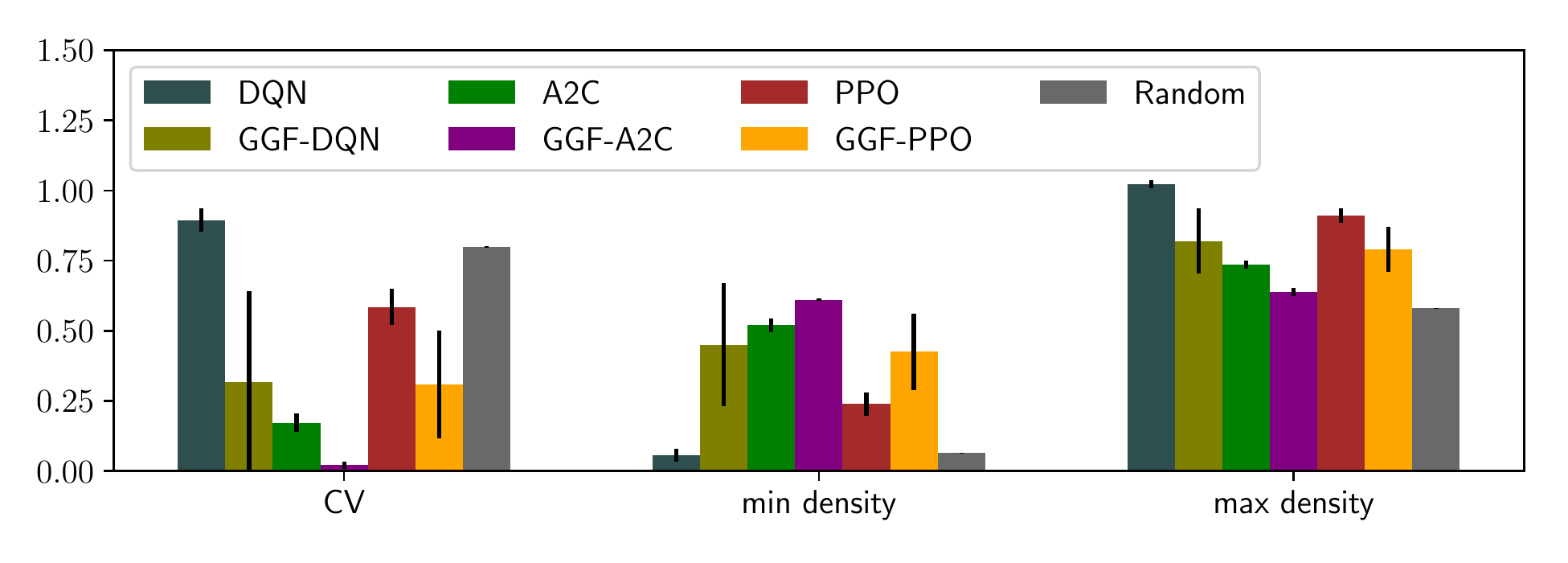}
\caption{CV, minimum and maximum densities of DQN, A2C, PPO and their GGF counterparts during the testing phase in the SC domain.} 
\label{fig:species-cv}
\end{center}
\end{figure}

\paragraph{Question (B)}
To answer (B), we turn to the TL domain, a more complex environment.
Figure~\ref{fig:online} shows the waiting times averaged over all the lanes obtained during the learning phase by the six different RL algorithms we considered.
As a reference, we added the performances of the random policy and a fixed policy.
This latter policy cycles between all the phases at a fixed frequency, which has been optimized over many simulations.
The fixed policy naturally performs better than the random one.
When training ends, it is worse than all RL algorithms. 
As expected, all standard algorithms performs better than their GGF counterparts, because the average accumulated waiting times correspond to the measure optimized by the original algorithms.

\begin{figure}[tb]
\begin{center}
\centering
\includegraphics[width=1.0\linewidth]{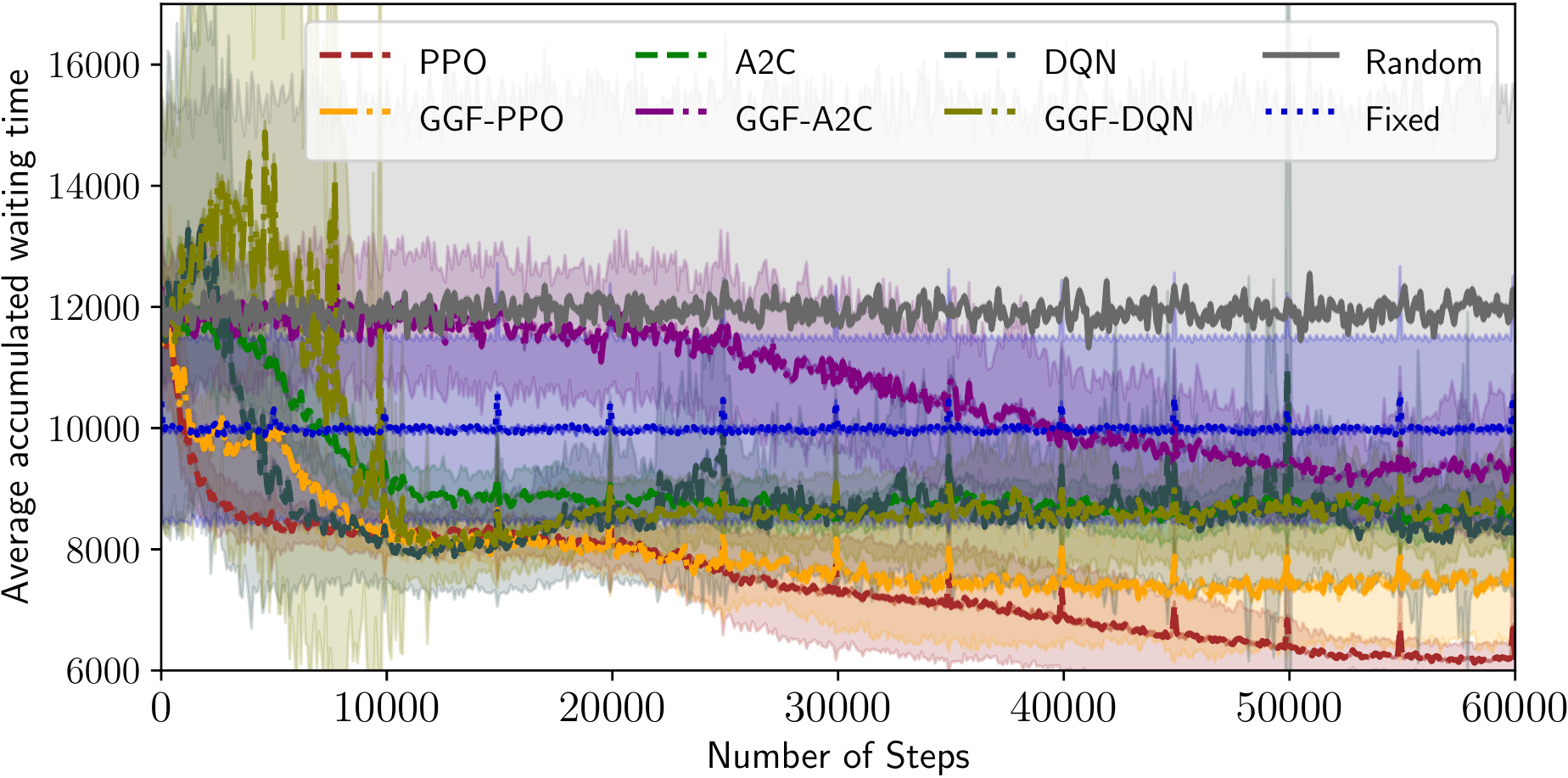}
\caption{Average waiting times of DQN, A2C, PPO, and their GGF counterparts during learning phase, and those of the fixed and random policies in the TL domain.}
\label{fig:online}
\end{center}
\end{figure}

Figure~\ref{fig:offline-performance-reward}  depicts the GGF score (over minus waiting times) computed after training. 
It also includes results for PPO and GGF-PPO with two different values of $\gamma$ (i.e., 0.99 and close to 1), which we discuss in the answer to (D).
Although, the fixed policy is not included for space reasons as it has the worse GGF-performance, we note that GGF-DQN performs better than the fixed policy, which shows that GGF-DQN has indeed optimized for fairness.
All the three GGF algorithms have better GGF scores than their original counterparts. 
GGF-PPO achieves the best score.

To confirm that those high GGF scores correspond to fairer solutions, we can indeed observe in Figure~\ref{fig:sumo-cv} that our proposed algorithms always achieve a lower CV than their original algorithms. 
Among all the algorithms, PPO performs the best as it has the lowest waiting time. 
Similar to PPO, GGF-PPO also has lower waiting times in all roads but with more balanced distributions of waiting times in each road.

\begin{figure}[tb]
\begin{center}
\centering
\includegraphics[width=1.0\linewidth]{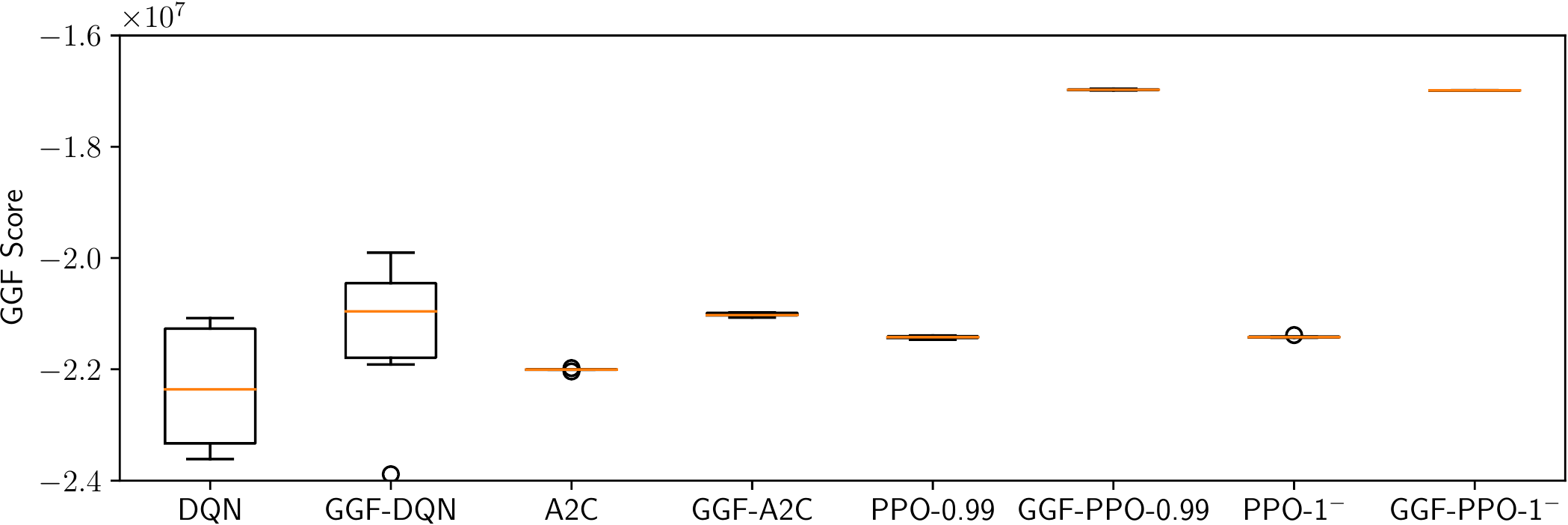}
\caption{GGF scores of DQN, A2C, PPO, and their GGF versions, with those of PPO and GGF-PPO  when $\gamma$ is close to 1, during the testing phase in the TL domain.}
\label{fig:offline-performance-reward}
\end{center}
\end{figure}

\begin{figure}[tb]
\begin{center}
\centering
\includegraphics[width=1\linewidth]{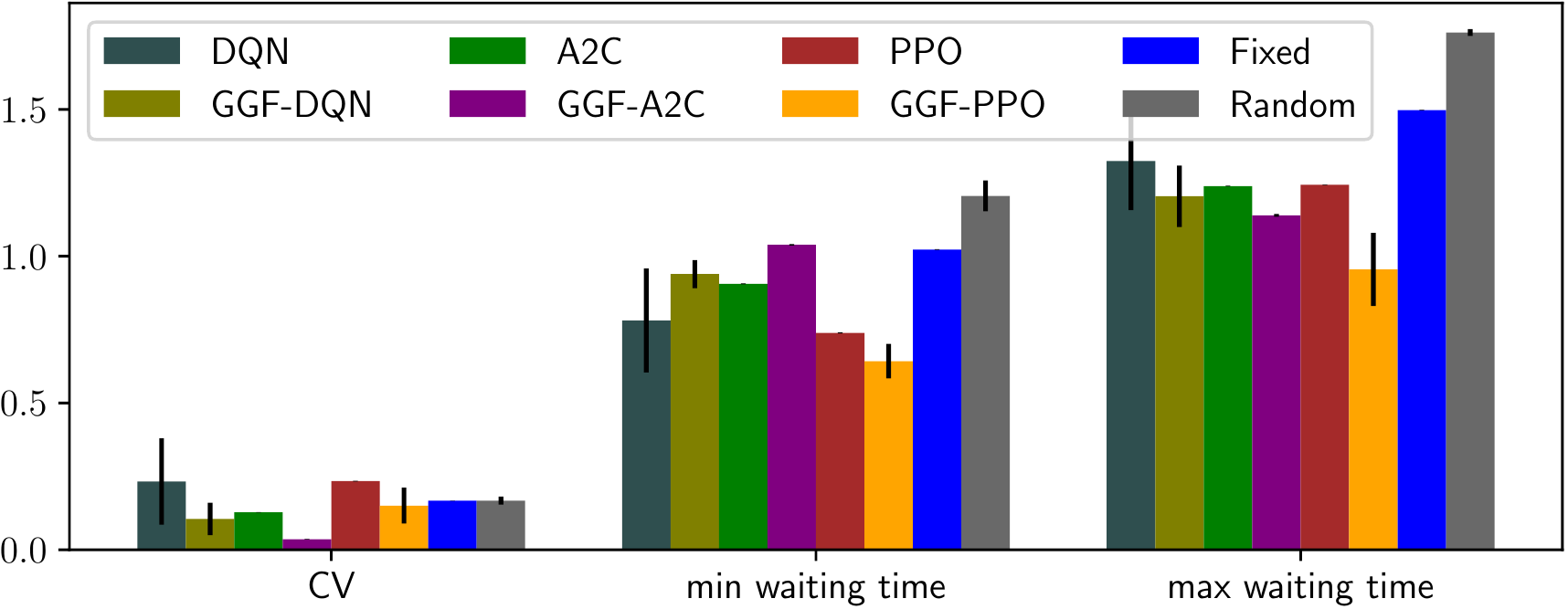}
\caption{CV, minimum and maximum waiting times of DQN, A2C, PPO and their GGF counterparts during the testing phase in the TL domain. The minimum and maximum waiting times have been divided by 3000 to be displayable with the CV.}
\label{fig:sumo-cv}
\end{center}
\end{figure}

\paragraph{Question (C)}

Algorithms optimizing a stochastic policy often perform better than DQN in terms of the average (or total) of the objectives (see~Figure~\ref{fig:species-ppo-accumulative}, Figure~\ref{fig:online}, or Figure~23 in Appendix).
However, in terms of GGF, for a simple domain like SC, GGF-DQN actually performs well, while the conclusion is reversed for a more complex domain such as TL.
This may be due to the partial observability of the domain and the use of GGF, which calls for stochastic policies for fairer solutions (as suggested by our theoretical discussion).
Figure~\ref{fig:online} indicates that the price of fairness (i.e., loss in terms of the average of the objectives for optimizing GGF instead of a utilitarian criterion) is limited.

\paragraph{Question (D)}
We also ran the algorithms with $\gamma$ very close to $1$, i.e., $\gamma=0.99999$.
The last two boxplots of Figure~\ref{fig:offline-performance-reward} show that the results for GGF-PPO are very similar to those with $\gamma=0.99$.
This indicates that the policy found by GGF-PPO is close to GGF-average optimal.
Also, this suggests that in practice, except for difficult MDP structures, using $\gamma=0.99$ is sufficient.

\paragraph{Question (E)}

Although our theoretical discussion concerned finite MOMDPs, we conjecture that similar results could be obtained in continuous spaces by adding some usual technical conditions \cite{Arapostathis1993}.
Therefore, we also tried our approach (with A2C and PPO) on the DC domain, whose states and actions are continuous.

Figure~\ref{fig:iroko-box} illustrates the offline performance of A2C, PPO, GGF-A2C, and GGF-PPO in terms of GGF score. 
As a reference, it also includes the GGF score of a fixed policy.
As expected, the fixed policy has the lowest GGF score.
GGF-A2C and GGF-PPO have better GGF scores than their original counterparts. 

This shows clearly that if a fair policy is important, the usual approach based on a weighted sum to aggregate the objectives (with equal weights) is insufficient.

\begin{figure}[tb]
\begin{center}
\centering
\includegraphics[width=1.0\linewidth]{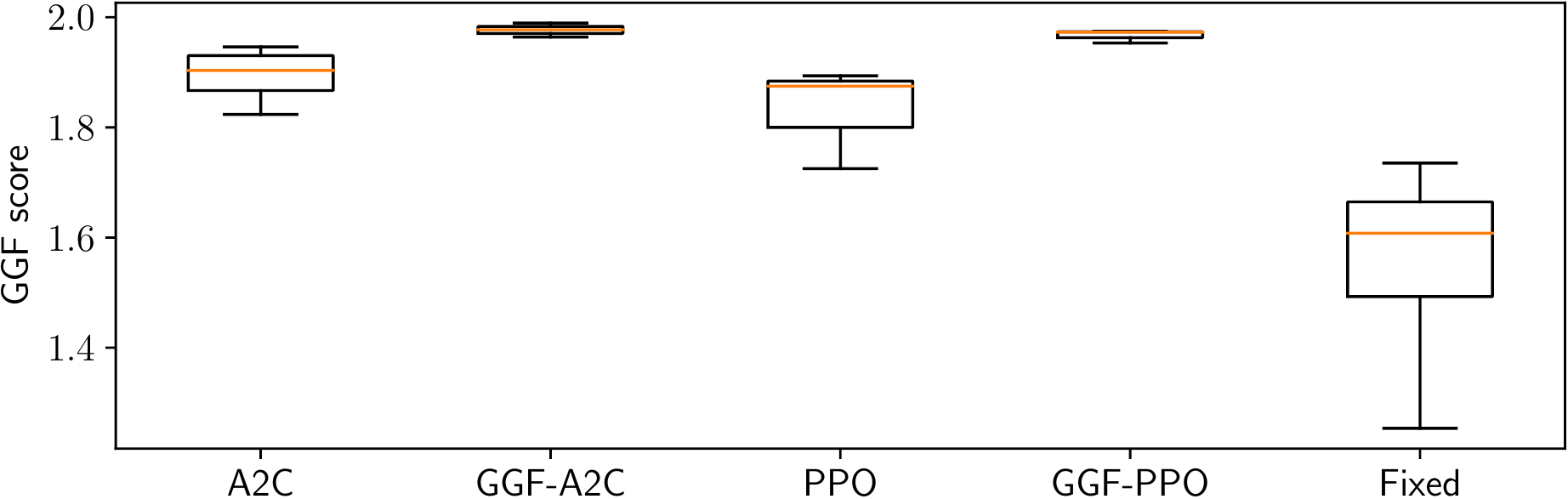}
\caption{GGF scores of A2C, PPO, and their GGF versions, with those of a fixed policy, during the testing phase in the DC domain.} 
\label{fig:iroko-box}
\end{center}
\end{figure}

Figure \ref{fig:iroko-cv} illustrates the performances of different RL algorithms and their GGF algorithms with those of the fixed and the random policy in terms of their CV, minimum and maximum of bandwidths.
As expected, the random policy performs worse as it has the lowest minimum and maximum bandwidths and also obtain the lowest accumulated bandwidth.
The fixed policy performs better than random one.
The GGF versions of A2C and PPO have lower CV which indicates that their are more fairer than their original algorithms.

For more experimental details and results, see the appendix.

%%%%%%%%%%%%%%%%%%%%%%%%%%%%%%%%%%%%%%%%%%%%%
\section{Related Work}\label{sec:relatedwork}

Fair optimization in applied mathematics, operations research, and theoretical computer science is an active research direction \citep{OgryczakLussPioroNaceTomaszewski14}.
Numerous classic continuous and combinatorial  optimization problems \citep{HurkalaSliwinski12,OgryczakPernyWeng13,NguyenWeng17,LescaMinouxPerny19} have been extended to take into account fairness.
However, these works assume that the whole model is known and therefore only focuses on the fair optimization problem since no learning is needed.
One notable work among them \citep{OgryczakPernyWeng13} solves Problem~\eqref{eq:ggi pb} with discounted rewards using a linear-programming approach, which is possible only if the transition and reward functions are known.
In contrast, in our work, we solve Problem~\eqref{eq:ggi pb} in the RL setting  considering both discounted and average rewards and also tackling problems with large or even continuous state space. 

Fairness considerations have recently become an important topic in machine learning \citep{Busa-FeketeSzorenyiWengMannor17,SpeicherHeidariGrgicHlacaGummadiSinglaWellerZafar18,AgarwalBeygelzimerDudikLangfordWallach18,HeidariFerrariGummadiKrause18,Jiang2019,Weng19}.
Most work focuses on the impartiality aspect of fairness.
However, a few notable exceptions consider fair optimization in sequential decision-making problems.
\citet{Busa-FeketeSzorenyiWengMannor17} investigate a problem similar to ours but in the multi-armed bandit setting.
\citet{Jiang2019} consider the problem of learning fair policies in multi-agent RL where fairness is defined over agents and encoded with a different welfare function.
Besides, their focus is on learning decentralized policies in a distributed way using a consensus mechanism.
Finally, another recent work deserves to be mentioned here although it does not specifically deal with fairness.
\citet{Cheung2019} investigates a problem more general than ours in the UCRL setting \citep{JakschOrtnerAuer2010} where the focus is on regret minimization for solving efficiently the exploration-exploitation dilemma.
That work also deals with non-stationary deterministic policies in tabular MDPs, contrary to our work.

%%%%%%%%%%%%%%%%%%%%%%%%%%%%%%%%%%%%%%%%%%%%%

\begin{figure}[tb]
\begin{center}
\includegraphics[width=1.0\linewidth]{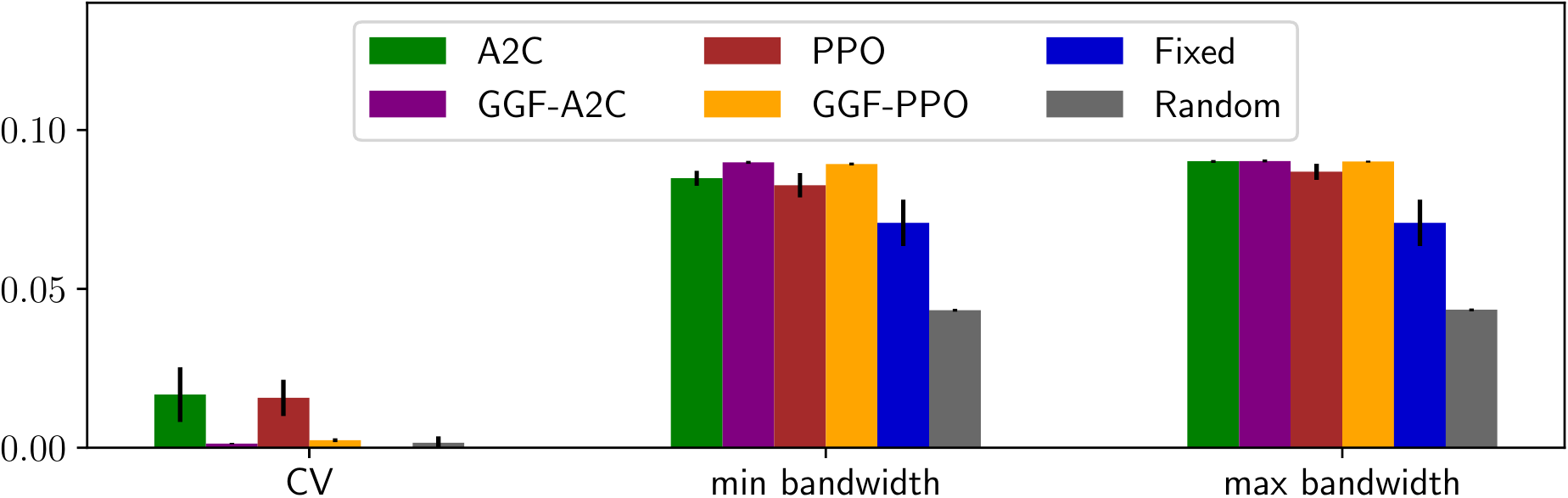}
\caption{CV, minimum and maximum bandwidths of A2C, PPO and their GGF counterparts during the testing phase in the DC domain.}
\label{fig:iroko-cv}
\end{center}
\end{figure}

\section{Conclusion}\label{sec:conclusion}

In this work, we introduced the novel problem of fair optimization in RL, which we theoretically discussed.
We proposed adaptations of three deep RL algorithms to solve large-scale problems and provided an extensive empirical validation.
As future work, we may consider other fair welfare functions \citep{OgryczakLussPioroNaceTomaszewski14}, extend to distributed control, or directly adapt RL algorithms for average reward.

\section*{Acknowledgements}

This work is supported in part by the program of National Natural Science Foundation of China (No. 61872238), the program of the Shanghai NSF (No. 19ZR1426700), and a Yahoo FREP grant.

\bibliographystyle{icml2020}
\bibliography{biblio190214}

\onecolumn

\icmltitle{Appendix: Learning Fair Policies in Multi-Objective Deep Reinforcement Learning with Average and Discounted Rewards}

\appendix

\section{Proofs of Theoretical Discussion} \label{sec:proof}

For better legibility, we first recall the equations and results that we need for our proofs.

\setcounter{equation}{7}
\begin{align}
    \GGF(\bm v) = \min_{\sigma \in \mathbb S_\nO} \w_\sigma^\intercal \bm v 
\end{align}
where $\mathbb S_\nO$ is the symmetric group of degree $\nO$ (i.e., set of permutations over $\{1, \ldots, \nO\}$), 
$\sigma$ is a permutation, and
$\w_\sigma = (\w_{\sigma(1)}, \ldots, \w_{\sigma(\nO)})$.

\setcounter{theorem}{4}
\begin{theorem} 
\citep{LamondPuterman89}
For any MDP, any stationary policy $\pi$, and any $\gamma \in (\frac{\sigma(\bm H_{\T_\pi})}{\sigma(\bm H_{\T_\pi})+1}, 1)$,
\vspace{-.5em}
\setcounter{equation}{9}
\begin{align}
    \vf_\pi = \frac{1}{1 - \gamma} \gain_{\pi} + \frac{1}{\gamma} \sum_{n=0}^{\infty} \left(\frac{\gamma-1}{\gamma}\right)^n \bm{H}_{\T_\pi}^{n+1} \R_\pi
\end{align}
where 
$\bm{H}_{\T_\pi}$ is the Drazin inverse of $\bm I - \T_\pi$, which is given by $(\bm I - \T_\pi + \T_\pi^*)^{-1}(\bm I - \T_\pi^*)$, 
$\T_\pi^*$ is the Ces\`aro-limit of $\T_\pi^n$ for $n \to \infty$, and
$\sigma(\bm H_{\T_\pi})$ is the spectral radius of matrix $\bm H_{\T_\pi}$.
\end{theorem}

\setcounter{equation}{10} 
\setcounter{theorem}{7}
We now present our proofs.

\begin{lemma}
For any MOMDP, Problem~\eqref{eq:ggi pb} admits a solution that is a stationary stochastic Markov policy.
\end{lemma}
\begin{proof}
For discounted rewards, a straightforward adaptation of the proof of Theorem~3.1 in \citep{Altman99} shows that the discounted occupation distribution of any policy can be obtained with that of a stationary stochastic Markov policy. 
As the evaluation of policies with GGF is completely determined by their discounted occupation distributions, the GGF of any policy can be obtained with that of a stationary stochastic Markov policy.

The situation is a bit more complicated for average rewards.
We recall two results for single-objective MDPs that can straightforwardly be extended to multi-objective MDPs.
Lemma~2.6 of \citet{Kallenberg2003} states that:
\begin{align}\label{eq:ineq}
    \lim_{\gamma \uparrow 1} (1-\gamma)\vvf_\pi \ge \vgain_\pi
\end{align}
for any policy $\pi$.
However, for such $\pi$, there exists a stationary policy $\pi^+$ such that for all $\gamma$ close to one, $(1-\gamma)\vvf_{\pi^+} \ge (1-\gamma)\vvf_\pi$.
Moreover, Corollary~2.5 of \citep{Kallenberg2003} states that  Inequality~\eqref{eq:ineq} becomes an equality for stationary policies.
Therefore, the set of Pareto-optimal gains can be obtained using only stationary policies.
\end{proof}

\begin{lemma}
For any weakly-communicating MOMDP, the GGF-average problem admits a solution that is a stationary stochastic Markov policy with constant gain.
\end{lemma}
\begin{proof}
We start by recalling the following property of weakly communicating MDP.
By Proposition~8.3.1 in \citep{Puterman94}, an MDP is weakly communicating if and only if there exists a stationary stochastic policy $\overline\pi$ which induces a Markov chain with a single closed irreducible class $\mathcal C$ and a set of states $\mathcal T$ that is transient under all stationary policies.

Lemma~\ref{lem:ssm} shows that we can focus on stationary stochastic Markov policies.
Assume by contradiction that all such policies that are solutions of \eqref{eq:ggi pb} with average rewards are such that their gains are not constant.
Let $\pi^*_1$ be such a policy.
By assumption, there exist two states $s_1$ and $s_2$ such that $\GGF(\vaR_{\pi^*_1, s_1}) > \GGF(\vaR_{\pi^*_1, s_2})$. 
Denote $\mathcal C_{s_1}$ the set of states containing $s_1$ that is closed, irreducible, and recurrent with respect to $\pi^*_1$.
We have necessarily $\mathcal C_{s_1} \subset \mathcal C$.

By the previous property of weakly-communicating MDPs, we can define a new policy $\pi^+$ such that $\forall s \in \mathcal C_{s_1}, \pi^+$ makes the same choices as $\pi^*_1$.
For all the other states, $\pi^+$ makes the same choices as $\overline\pi$. 

By definition, $\pi^+$ induces a Markov chain with a single closed irreducible class $\mathcal C_{s_1}$ and a set of transient states $\St \backslash \mathcal C_{s_1}$.
The gain of policy $\pi^+$ is therefore constant and equal to $\vaR_{\pi^*_1, s_1}$, which contradicts our previous assumption.
Therefore, the lemma holds.
\end{proof}

In the next theorem, $\pi^*_\gamma$ is GGF-$\gamma$-optimal and $\pi^*_1$ is GGF-average-optimal.
\setcounter{theorem}{5}
\begin{theorem}
For any weakly-communicating MOMDP and any $\gamma \in (\max(\frac{\sigma(\bm H_{\T_{\pi^*_\gamma}})}{\sigma(\bm H_{\T_{\pi^*_\gamma}})+1}, \frac{\sigma(\bm H_{\T_{\pi^*_1}})}{\sigma(\bm H_{\T_{\pi^*_1}})+1}, 1)$,
\begin{align*}
&\GGF(\vaR_{\pi^*_\gamma}) 
     \ge \GGF(\vaR_{\pi^*_1}) - \overline\vR (1-\gamma) \left(
     \rho(\gamma, \sigma(\bm{H}_{\T_{\pi^*_1}})) + \rho(\gamma, \sigma(\bm{H}_{\T_{\pi^*_\gamma}})) \right)
\end{align*}
where $\overline\vR = \max_\pi \|\vR_\pi\|_1$,  
$\sigma(\bm M)$ is the spectral radius of matrix $\bm M$, and 
$\rho(\gamma, \sigma) = \frac{\sigma}{\gamma -(1-\gamma) \sigma}$.
\end{theorem}
\begin{proof}
We start with some notations. For any permutation $\sigma$, let $\mathcal M^\sigma$ be the MDP obtained from the initial MOMDP with the reward function defined by $\tilde\R^\sigma_{a,s} = \w_\sigma^\intercal \vR_{a,s}$.
Naturally, the MDP and MOMDP have the same policies.
An element (e.g., optimal policy, value function, gain) corresponding specifically to $\mathcal M^\sigma$ will be marked with the tilde sign and exponent $\sigma$, e.g., $\tilde \vf^\sigma_\pi$ (resp. $\tilde \gain^\sigma_\pi$) is the value function (resp. gain) of policy $\pi$ in $\mathcal M^\sigma$.

By \eqref{eq:ggi min}, there exists $\sigma$ such that $\GGF(\vaR_{\pi^*_\gamma}) = \w_\sigma^\intercal \vaR_{\pi^*_\gamma}$.
We now make two observations regarding $\pi^*_\gamma$ and $\pi^*_1$.
Regarding $\pi^*_\gamma$, we have:
\begin{align}
%\w_\sigma^\intercal \vvf_{\pi^*_\gamma} 
 \w_\sigma^\intercal (\sd_0 \vvf_{\pi^*_\gamma})^\intercal
    &= \sd_0 \tilde{\vf}^\sigma_{\pi^*_\gamma} \label{eq:pistargamma1}\\
%    &= \tilde{\vf}^{\tilde{\pi}^*_\gamma} \label{eq:pistargamma2}
    & \ge \GGF\big ((\sd_0 \vvf_{\pi^*_\gamma})^\intercal\big) \label{eq:pistargamma2}\\
    & \ge \GGF\big((\sd_0 \vvf_{\pi^*_1})^\intercal \big) \label{eq:pistargamma3}\\
    &= \w_{\sigma'}^\intercal (\sd_0 \vvf_{\pi^*_1})^\intercal  \label{eq:pistargamma4}\\
    & = \sd_0 \tilde{\vf}^{\sigma'}_{\pi^*_1} \label{eq:pistargamma5}
\end{align}
where 
\eqref{eq:pistargamma1} and \eqref{eq:pistargamma5} are obtained by linearity,
%\eqref{eq:pistargamma2} holds by definition of $\pi^*_\gamma$ (i.e., it maximizes GGF, which is equivalent here to the weighted sum with $\w_\sigma$).
\eqref{eq:pistargamma2} holds by \eqref{eq:ggi min}, 
\eqref{eq:pistargamma3} is true by definition of the two policies, and
\eqref{eq:pistargamma4} holds for some $\sigma'$.
Regarding $\pi^*_1$, we have:
\begin{align} 
    \GGF(\vaR_{\pi^*_1}) 
    \le \w_{\sigma'}^\intercal \vaR_{\pi^*_1} 
    %= \tilde{\aR}^{\sigma'}_{\pi^*_1}
    = \sd_0\tilde{\gain}^{\sigma'}_{\pi^*_1} \label{eq:pistarone}
\end{align}
where the inequality comes from \eqref{eq:ggi min} and the equality is obtained by linearity.

Using the first observation, we obtain:
\begin{align}
\sd_0 \tilde{\gain}^\sigma_{\pi^*_\gamma} 
     &= \sd_0 \big( (1 - \gamma) \tilde{\vf}^\sigma_{\pi^*_\gamma}  + \sum_{n=1}^{\infty} \left(\frac{\gamma-1}{\gamma}\right)^{n}  \bm{H}_{\T_{\pi^*_\gamma}}^{n} \tilde{\R}^\sigma_{\pi^*_\gamma} \big) \label{eq:thm1}\\
     &\ge (1 - \gamma) \sd_0 \tilde{\vf}^{\sigma'}_{\tilde{\pi}^*_1} + \sum_{n=1}^{\infty} \sd_0\left(\frac{\gamma-1}{\gamma}\right)^n  \bm{H}_{\T_{\pi^*_\gamma}}^{n} \tilde{\R}^\sigma_{\pi^*_\gamma} \label{eq:thm2}\\
%     &\ge (1 - \gamma) \tilde{\vf}_{\pi^*_1} + \sum_{n=1}^{\infty} \left(\frac{\gamma-1}{\gamma}\right)^n  \bm{H}_{\T_{\pi^*_\gamma}}^{n} \tilde{\R}_{\pi^*_\gamma} \label{eq:thm3}\\
     &= \sd_0\tilde{\gain}^{\sigma'}_{\pi^*_1} - \sum_{n=1}^{\infty} \left(\frac{\gamma-1}{\gamma}\right)^n \sd_0 \bm{H}_{\T_{\pi^*_1}}^{n} \tilde{\R}^{\sigma'}_{\pi^*_1} +  \sum_{n=1}^{\infty} \left(\frac{\gamma-1}{\gamma}\right)^n  \sd_0 \bm{H}_{\T_{\pi^*_\gamma}}^{n} \tilde{\R}^{\sigma}_{\pi^*_\gamma} \label{eq:thm4}\\
     &%\tilde{\aR}^{\sigma}_{\pi^*_\gamma} 
\ge \GGF(\vaR_{\pi^*_1})
- \sum_{n=1}^{\infty} \left(\frac{\gamma-1}{\gamma}\right)^n \sd_{0} \bm{H}_{\T_{\pi^*_1}}^{n} \tilde{\R}^{\sigma'}_{\pi^*_1} + \sum_{n=1}^{\infty} \left(\frac{\gamma-1}{\gamma}\right)^n  \sd_{0} \bm{H}_{\T_{\pi^*_\gamma}}^{n} \tilde{\R}^\sigma_{\pi^*_\gamma} \label{eq:thm5}
\end{align}
where \eqref{eq:thm1} is obtained from \eqref{eq:decomposition} applied to $\tilde{\pi}^*_\gamma$ and rearranging terms, 
\eqref{eq:thm2} comes from \eqref{eq:pistargamma2}, 
%\eqref{eq:thm3} holds because $\tilde{\pi}^*_\gamma$ is $\gamma$-optimal in $\mathcal M$, and
\eqref{eq:thm4} is obtained from \eqref{eq:decomposition} applied to $\pi^*_1$, and
\eqref{eq:thm5} holds because the gain of $\pi^*_1$ is constant by Lemma~\ref{lem:constant gain}.

We now show how the second term in the right-hand side of the inequality can be bounded:
\begin{align}
    \sum_{n=1}^{\infty} \left(\frac{\gamma-1}{\gamma}\right)^n \sd_{0} \bm{H}_{\T_{\pi^*_1}}^{n} \tilde{\R}^{\sigma'}_{\pi^*_1}
    & \le \sum_{n=1}^{\infty} \left(\frac{1 - \gamma}{\gamma}\right)^n \|\sd_{0}\| \|\bm{H}_{\T_{\pi^*_1}}^{n}\|_2 \|\tilde{\R}^{\sigma'}_{\pi^*_1}\| \label{eq:cor4}\\
    & \le \sum_{n=1}^{\infty} \left(\frac{1 - \gamma}{\gamma}\right)^n \|\bm{H}_{\T_{\pi^*_1}}\|_2^n \|\tilde{\R}^{\sigma'}_{\pi^*_1}\| \label{eq:cor5}\\
    & \le \sum_{n=1}^{\infty} \left(\frac{1 - \gamma}{\gamma}\right)^n \|\bm{H}_{\T_{\pi^*_1}}\|_2^n \overline\vR \label{eq:cor6}\\
    & = \sum_{n=1}^{\infty} \left(\frac{1 - \gamma}{\gamma}\right)^n \sigma(\bm{H}_{\T_{\pi^*_1}})^n \overline\vR \nonumber \\
    & = \overline\vR \left(\frac{1}{1 - \frac{1-\gamma}{\gamma}\sigma(\bm{H}_{\T_{\pi^*_1}})} - 1\right)  \label{eq:cor8}\\
    & = \overline\vR  \frac{(1-\gamma)\sigma(\bm{H}_{\T_{\pi^*_1}})}{\gamma - (1-\gamma)\sigma(\bm{H}_{\T_{\pi^*_1}})} \nonumber
\end{align}
where \eqref{eq:cor4} is obtained by applying the Cauchy-Schwartz inequality, 
\eqref{eq:cor5} uses $\|\sd_{0}\| \le 1$,
\eqref{eq:cor6} uses $\|\w\|_1 = 1$, 
and
\eqref{eq:cor8} holds by assumption ($\gamma > \sigma(\bm{H}_{\T_{\pi^*_1}})/(\sigma(\bm{H}_{\T_{\pi^*_1}})+1)$).

A similar bound can be found for the third term.
Finally, the result follows by plugging the two bounds in \eqref{eq:thm5} and the fact that $\GGF(\vaR_{\pi^*_\gamma}) = \w_\sigma^\intercal \vaR_{\pi^*_\gamma} = \sd_0 \tilde{\gain}^\sigma_{\pi^*_\gamma} %\tilde{\aR}_{\pi^*_\gamma}
$, which holds by linearity.
\end{proof}

%%%%%%%%%%%%%%%%%%%%%%%%%%%%%%%%%%%%%%%%%%%%%
\section{Descriptions of Experimental Domains}
\subsection{Conservation of two endangered species}

This domain is based on the model introduced by \citet{chades2012setting} that describes the interaction of two endangered species, sea otters and its prey, northern abalones.
In this problem, interventions can be taken in order to maintain the populations of the two species at relatively balanced levels.

The environment for this problem is in fact a partially observable MOMDP and similarly to \citep{chades2012setting}, we solve it as an MOMDP.
The model can be summarized as follows\footnote{For simplicity, we use similar notations as in \citep{chades2012setting} since there is no much risk of confusion with those used in our main paper.} (motivation and more explanation can be found in \citep{chades2012setting}):
\begin{itemize}
    \item At time step $t$, a state consists of 
    $N^O_t$ (the population number of sea otters), 
    $\bm N^{A}_{t}$ (a 10-dimensional vector indexed from 4 to 13, where $\bm N^{A}_{i,t}$ is the population number of abalones for age group $i$), and a $10 \times 10 $ matrix representing the survival rate of abalones for different age and living area.
    In the model, $10$ age groups are considered from $4$ to $13$:
    the enrollment age starts from age $4$ and all the ages greater than $13$ are pooled together into the $13$ age group.
    The initial state is fixed to some stable abalone population numbers.
    \item An observation is defined as a pair $ (n_t^O,n_t^A) \in \{1, ..., 21\} \times \{1, ..., 39\}$ where $n_t^O$ (respectively $n_t^A$) is a discretization of $N_t^O$ to represent sea otters (respectively $\sum_i \bm N_{i,t}^A$ to represent abalones).
    \item Five management actions are considered: \textit{do nothing}, \textit{introduce sea otters}, \textit{enforce abalone antipoaching measures}, \textit{control sea otters}, and \textit{one-half antipoaching and one-half control sea otters}. 
          
    \item The transition function is based on the population growth models of the two species taking into account factors such as poaching and predation (for abalones) or oil spills (for sea otters).
    The next state is computed in the following order : 1) apply abalone and sea otter growth models independently of the action, 2) potentially apply culling of sea otters according to the action, 3) remove abalone because of predation, 4) remove abalone because of poaching according to the action. 
    The interaction between sea otters and abalone is modeled with a linear function response.

    \item 
    The reward after performing action $a$ in state $s_t$ is a vector consisting of two components:
    $$\vR_{a,s_t} = \begin{bmatrix} JR_{so}(N^O_t) \\ 
    JR_{aba}(\sum_{i=3}^{14} \bm N^{A}_{i,t}) \end{bmatrix}$$
    where $s_t = (N_t, d^3_t, \ldots, d^{14}_t)$, and the two functions $JR_{so}$ and $JR_{aba}$ are introduced to make a population number of sea otters and a density of northern abalones commensurable. 
    Note that rewards do not depend on actions here.
    In \cite{chades2012setting}, a scalar reward was defined as the minimum of those two components in order to balance the two objectives.
\end{itemize}

\subsection{Traffic Light Control}

We also evaluate our method in the classic traffic light control problem.
While the usual approach to this problem consists in minimizing the expected waiting times averaged over all lanes, we consider the expected waiting times of the four directions (north, east, south, or west) at the intersection separately.
We use Simulation of Urban MObility (SUMO)\footnote{\url{https://github.com/eclipse/sumo}} to simulate a single eight-lane intersection (see Figure~\ref{fig:sumo-env}).
Depending on intersections, different numbers of traffic light phases can be considered.
A traffic light phase specifies which lanes have the green light.
We assume here that there are four phases $NSL$, $NSSR$, $EWL$, and $EWSR$.
Phase $NSL$ (North-South Left) corresponds to the case where the green light is given to cars in the left lanes of the roads coming from the north and south.
The cars can only turn left in this phase. 
Phase $NSSR$ (North-South Straight and Right) allows cars in the right lanes for the north-south axis to go straight or turn right.
Phases $EWL$ and $EWSR$ are defined similarly for the east-west axis.

\begin{figure}[h]
\vskip 0.2in
\begin{center}
\centering
\includegraphics[width=0.3\linewidth]{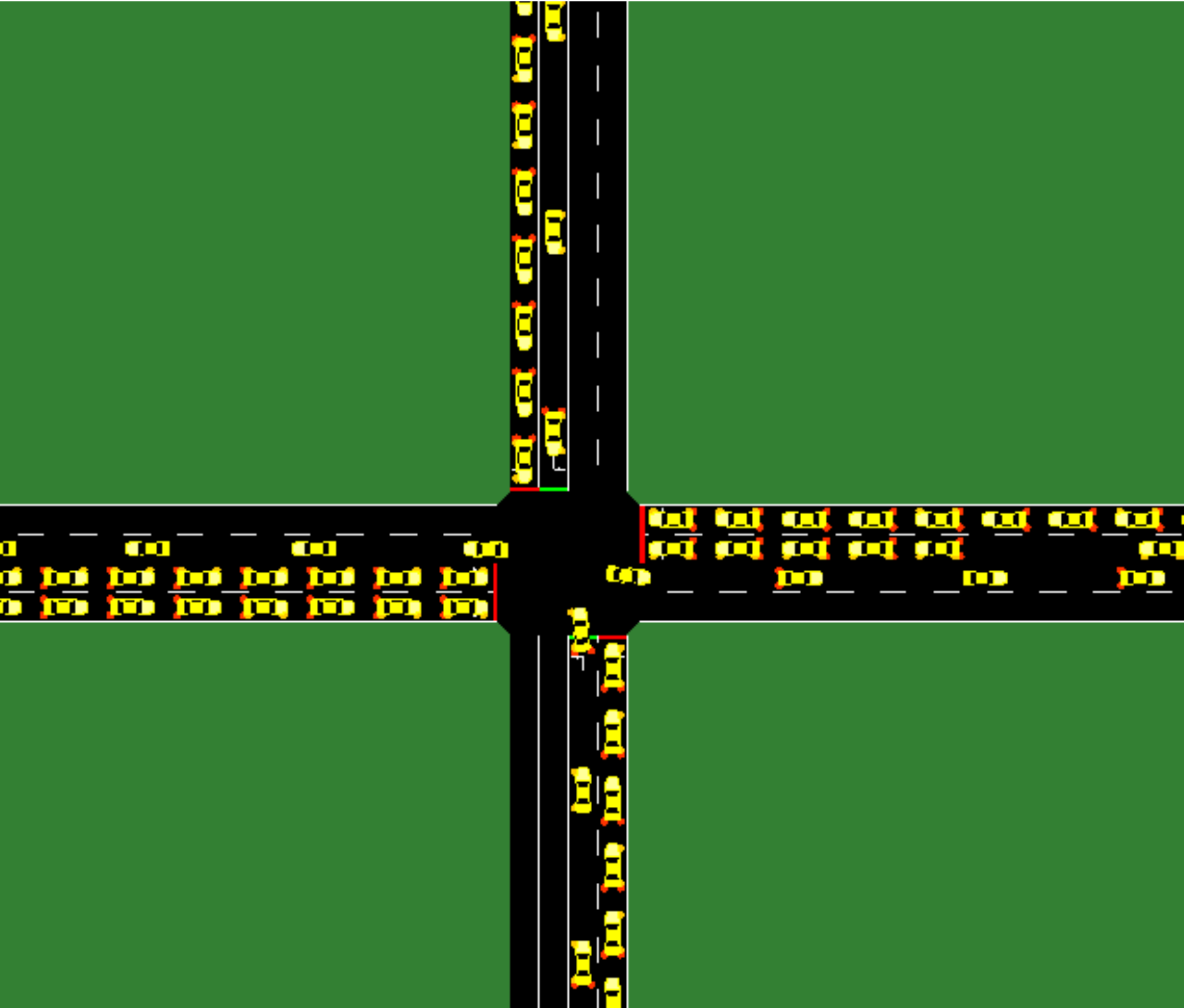}
\caption{SUMO traffic controller simulation.}
\label{fig:sumo-env}
\end{center}
\vskip -0.2in
\end{figure}

We formulate the MOMDP for this domain as follows:
\begin{itemize}
    \item A state is a 20-dimensional vector containing the current traffic light phase (4 phases represented with a one-hot encoding) and for each lane, its total waiting time and density of cars stopped at the intersection ($8 \times 2$).

    \item 
    An action corresponds to a traffic light phase. 
    Therefore, 
    $ \Ac = \{NSL, NSSR, EWL, EWSR\} $.

    \item The transition function depends on 
    the current traffic light phase, 
    how cars drive through the intersection, and 
    how new traffic is generated.
    
    The duration of each phase is fixed (which corresponds to one time step in the RL problem). 
    The green light time for each phase is 10 seconds while the yellow time is 4 (if there is a change to a new phase).
    For simplicity, all the vehicles have the same  characteristics (e.g., car speed, acceleration, length) that are provided by default in SUMO.
    The phase duration and the car characteristics determine how many waiting cars can drive through the intersection.
    
    At each time step, for each lane, new vehicles enter the intersection  according to a fixed Bernoulli distribution in each episode. 
    The probabilities used for each lane are provided in Figure~\ref{fig:dist1}.
    They simulate a heavy traffic intersection.
    \begin{figure}
    \vskip 0.2in
    \begin{center}
    \centering
    \includegraphics[width=0.5\linewidth]{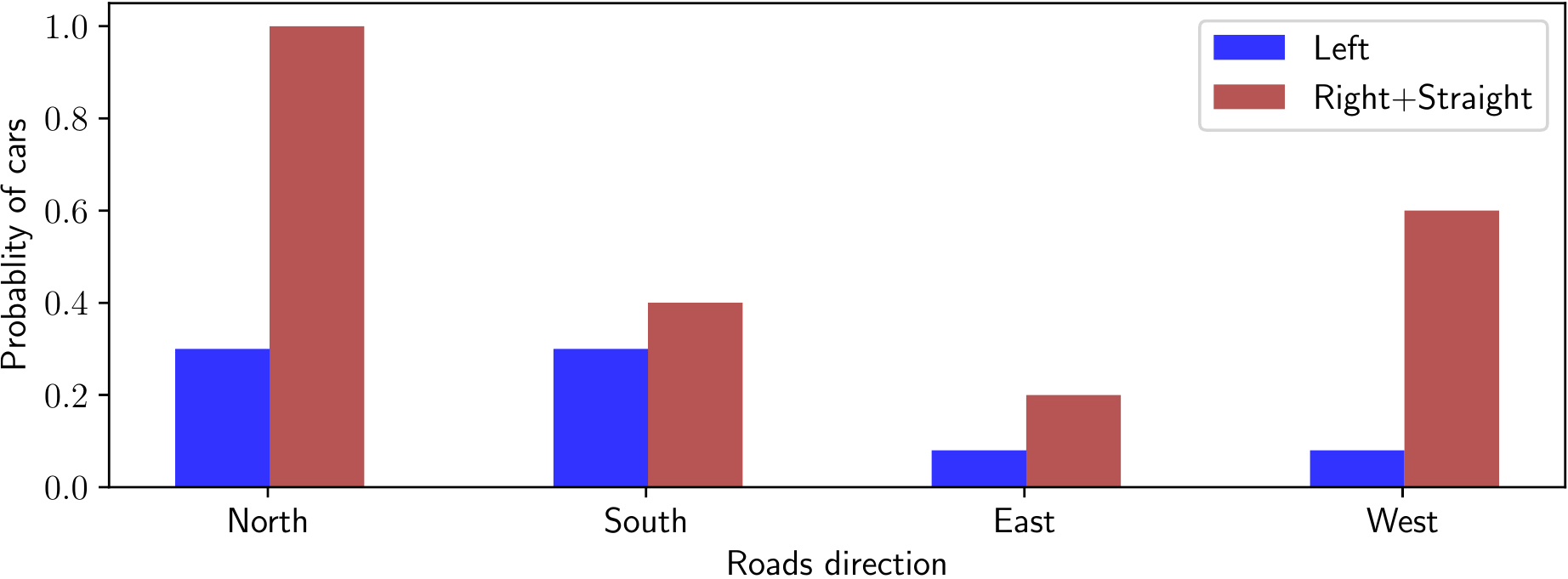}
    \caption{Probabilities of cars entering in each lane.}
    \label{fig:dist1}
    \end{center}
    \vskip -0.2in
    \end{figure}
    
    \item 
    A reward is a vector of 4 components corresponding to the waiting times of each direction:
    $$\vR_{a,s} = 
    \begin{bmatrix} 
        -\sum_{j=1}^{2} {W}^1_j(s) \\ 
        -\sum_{j=1}^{2} {W}^2_j(s) \\
        -\sum_{j=1}^{2} {W}^3_j(s) \\
        -\sum_{j=1}^{2} {W}^4_j(s)
    \end{bmatrix}$$
    where ${W}^i_j(s)$ is the total waiting time of all the cars of the $i^{th}$ direction and $j^{th}$ lane.
    The standard approach to this problem would define the scalar reward as the sum of those components.
\end{itemize}

\subsection{Data Center Traffic Control}

In the Data Center (DC) traffic congestion control problem  \cite{iroko}, a centralized controller manages a  computer network that is shared by a certain number of hosts in order to optimize the bandwidths of each host.
For the network topology, a fat-tree topology (see Figure~\ref{fig:fat-tree} is considered, which has $D=16$ hosts, $20$ switches with $n=4$ ports each, which leads to a total of $80$ queues.
For the experiments, we used Mininet\footnote{\url{https://github.com/mininet/mininet}} to simulate the network with the fat-tree topology using UDP as the underlying transport protocol and goben\footnote{\url{https://github.com/udhos/goben}} to generate traffic and monitor/collect network information. 
In \cite{iroko}, a reward function was designed such that an RL agent would learn to maximize a sum of host bandwidths penalized by queue lengths (in order to avoid switch bufferbloats).
Here, we instead aim at maximizing separately the bandwidth of each host penalized by queue lengths, while ensuring fairness.
\begin{figure}[h]
\vskip 0.2in
\begin{center}
\centering
\includegraphics[width=0.5\linewidth]{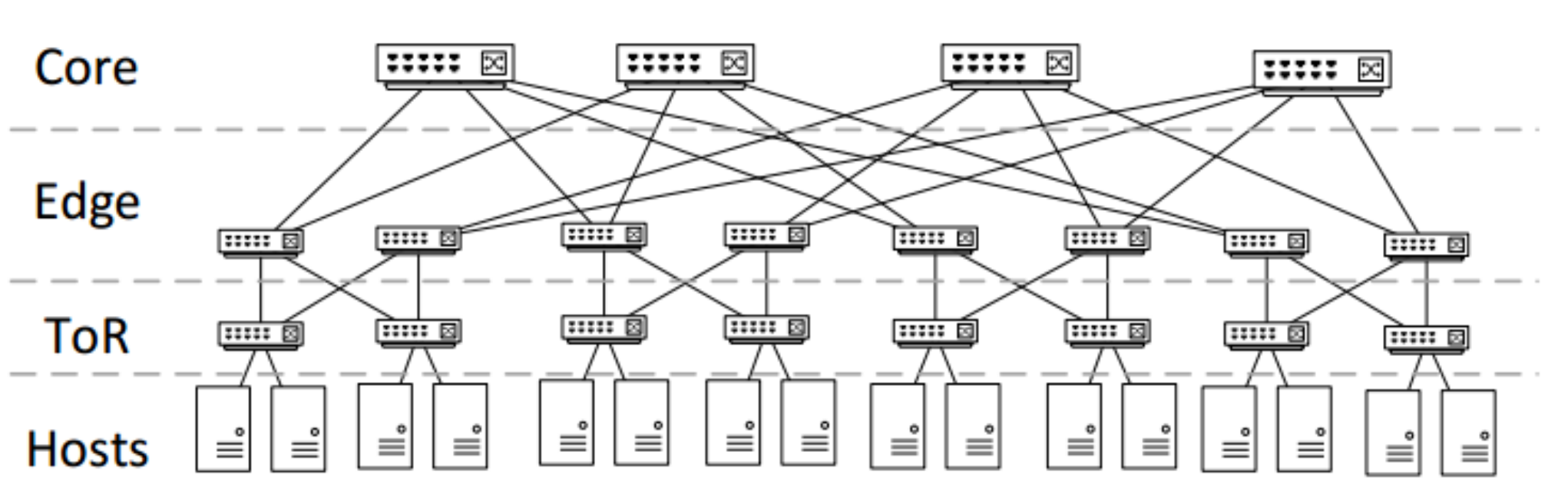}
\caption{Network with a fat-tree topology from \cite{iroko}.}
\label{fig:fat-tree}
\end{center}
\vskip -0.2in
\end{figure}

The MOMDP for this problem can be summarized as follows (for motivation and more explanation check \citep{iroko}):
\begin{itemize}
    \item A state is composed of a $n\times m$ matrix that stores the information statistics from the transport and lower layers,  where $n$ is the number of ports in a switch and $m$ is the number of network features (i.e., the queue length, the derivative over time of the queue length, the number of packet drops and the queue length above the limit).
    A state also contains information about the current bandwidth allocation to the $D$ hosts.
    A bandwidth is a value in $[0, 10]$.
    
    \item An action is a $D$-dimensional vector of bandwidth allocation to the hosts.
    
    \item  
    Network traffic between hosts is generated according to an input file (see \cite{iroko}).
    
    \item
    The $D$-dimensional reward vector is defined as follows:
    $$\vR_{\bm a,s} = \bm{a} - 2*\bm{a}*\max_{i}q_{i}(s)$$
    where $\bm a$ is the vector action that represents the bandwidth allocation and  $q_i(s) $ represents the $i$-th queue length.
    This definition is adapted from \cite{iroko}, which uses the average of the previous components to define a scalar reward.
\end{itemize}

%%%%%%%%%%%%%%%%%%%%%%%%%%%%%%%%%%%%%%%%%%%%%
\section{Hyperparameters}
For the sake of reproducibility, all the hyperparameters for all the environments are reported in Tables~\ref{table:hyperpara-PPO},  \ref{table:hyperpara-DQN} and \ref{table:hyperpara-A2C}.
In those tables, subscripts "sc", "tl", "dc" stand for our three experimental domains, respectively, species conservation, traffic lights control, and data center traffic control. 

\begin{table}[H]
\caption{Set of hyperparameters used during training with PPO and GGF-PPO}
\label{table:hyperpara-PPO}
\vskip 0.15in
\begin{center}
\begin{small}
\begin{sc}
\begin{tabular}{lcccr}
\toprule
Hyperparameter & Values$_{PPO}$ & Values$_{GGF-PPO}$ \\
\midrule
$\gamma$    & $0.99_{sc, tl, dc}$ & $0.99_{sc, tl, dc}$ \\
Learning rate & $0.001_{sc},0.0005_{tl}, 0.0001_{dc}$  & $0.00005_{sc, dc}, 0.0005_{tl}$ \\
n\_envs     & $ 10_{sc, tl}, 1_{dc} $     & $ 10_{sc, tl}, 1_{dc} $ \\
n\_steps per update     & $ 128_{sc, tl, dc} $     & $ 128_{sc, tl, dc} $ \\
cliprange   & $0.2_{tl}, 0.1_{sc}, 0.5_{dc} $   & $0.2_{tl}, 0.1_{sc}, 0.5_{dc}$   \\
Actor Network    & $64 * 64_{sc, tl, dc}$  & $64 * 64_{sc, tl, dc}$ \\
Critic Network    & $64 * 64_{sc, tl, dc}$  & $64 * 64_{sc, tl, dc}$ \\
Networks hidden activation     & $ TanH_{sc, tl, dc} $  & $ TanH_{sc, tl, dc} $ \\
Networks output activation     & $ Linear_{sc, tl, dc} $   & $ Linear_{sc, tl, dc} $ \\
Optimizer     & $ Adam_{sc, tl, dc} $     & $ Adam_{sc, tl, dc} $\\
Adam Epsilon     & $ 1e^{-5}_{sc, tl, dc} $     & $ 1e^{-5}_{sc, tl, dc} $ \\
ent\_coef     & $ 0.01_{sc, tl, dc} $  & $ 0.01_{sc, tl, dc} $ \\
vf\_coef     & $ 0.5_{sc, tl, dc} $   & $ 0.5_{sc, tl, dc} $ \\
\bottomrule
\end{tabular}
\end{sc}
\end{small}
\end{center}
\vskip -0.1in
\end{table}

\begin{table}[H]
\caption{Set of hyperparameters used during training with A2C and GGF-A2C}
\label{table:hyperpara-A2C}
\vskip 0.15in
\begin{center}
\begin{small}
\begin{sc}
\begin{tabular}{lcccr}
\toprule
Hyperparameter & Values$_{A2C}$ & Values$_{GGF-A2C}$ \\
\midrule
$\gamma$    & $0.99_{sc, tl, dc}$ & $0.99_{sc, tl, dc}$ \\
Learning rate & $0.0001_{sc, tl, dc}$  & $0.0001_{sc, tl}, 0.0005_{dc}, $ \\
n\_envs     & $ 10_{sc, tl}, 1_{dc} $     & $ 10_{sc, tl}, 1_{dc} $ \\
n\_steps per update     & $ 10_{sc}, 5_{tl, dc} $     & $ 30_{sc, tl, dc} $ \\
Actor Network    & $64 * 64_{sc, tl, dc}$  & $64 * 64_{sc, tl, dc}$ \\
Critic Network    & $64 * 64_{sc, tl, dc}$  & $64 * 64_{sc, tl, dc}$ \\
Networks hidden activation     & $ TanH_{sc, tl, dc} $  & $ TanH_{sc, tl, dc} $ \\
Networks output activation     & $ Linear_{sc, tl, dc} $   & $ Linear_{sc, tl, dc} $ \\
Optimizer     & $ RMSprop_{sc, tl, dc} $     & $ RMSprop_{sc, tl, dc} $\\
RMSProp Epsilon     & $ 1e^{-5}_{sc, tl, dc} $     & $ 1e^{-5}_{sc, tl, dc} $ \\
RMSProp alpha     & $ 0.99_{sc, tl, dc} $     & $ 0.99_{sc, tl, dc} $ \\
value\_func\_coef     & $ 0.25_{sc, tl, dc} $  & $ 0.25_{sc, tl, dc} $ \\
entropy \_coef     & $ 0.01_{sc, tl, dc} $   & $ 0.01_{sc, tl, dc} $ \\
\bottomrule
\end{tabular}
\end{sc}
\end{small}
\end{center}
\vskip -0.1in
\end{table}

\begin{table}[H]
\caption{Set of hyperparameters used during training with DQN and GGF-DQN}
\label{table:hyperpara-DQN}
\vskip 0.15in
\begin{center}
\begin{small}
\begin{sc}
\begin{tabular}{lcccr}
\toprule
Hyperparameter & Values$_{DQN}$ & Values$_{GGF-DQN}$ \\
\midrule
$\gamma$    & $0.99_{sc, tl}$ & $0.99_{sc, tl}$ \\
Learning rate & $0.0001_{sc}, 0.0005_{tl}$ & $0.005_{sc}$, $0.0005_{tl}$ \\
batch\_size     & $64_{sc}$, $128_{tl} $ & $128_{sc, tl} $ \\
Q Network    & $64 * 64_{sc, tl}$ & $64 * 64_{sc, tl}$ \\
Network hidden activation     & $ ReLU_{sc, tl} $   & $ ReLU_{sc, tl} $ \\
Network output activation     & $ ReLU_{sc, tl} $  & $ ReLU_{sc, tl} $ \\
buffer\_size     & $ 50000_{sc, tl} $  & $ 50000_{sc, tl} $ \\
exploration\_fraction     & $ 0.1_{sc, tl} $  & $ 0.1_{sc, tl} $ \\
prioritized\_replay     & $ False_{sc, tl} $    & $ False_{sc, tl} $ \\
exploration\_fraction     & $ 0.1_{sc, tl} $  & $ 0.1_{sc, tl} $ \\
target\_network\_update\_freq     & $ 500_{sc, tl} $  & $ 500_{sc, tl} $ \\
Dueling     & $ False_{sc, tl} $    & $ False_{sc, tl} $ \\
\bottomrule
\end{tabular}
\end{sc}
\end{small}
\end{center}
\vskip -0.1in
\end{table}

%%%%%%%%%%%%%%%%%%%%%%%%%%%%%%%%%%%%%%%%%%%%%
\section{Additional Experimental Results}
\label{sec:expes_add}

We present here additional experimental results that were not included in the main document.

\subsection{Conservation of two endangered species}

In the SC domain we also perform some weight analysis. 
As it is a small problem with two objectives, it is easy to perform those experiments and clearly show how our approach yields more balanced or fairer solutions. 
In addition to the experiments that we have shown in the main paper, we also performed experiments where the GGF weights are decreasing faster.
Concretely, GGF coefficients were defined as $\w_i = \frac{1}{2^i}$ from 0 to $\nO-1$, while in this set of experiments it is defined as $\w_i = \frac{1}{10^i}$ from 0 to $\nO-1$.
Similar conclusions, which we detail next, can be drawn with both sets of weights.

Figure~\ref{fig:species-box10} shows the distributions of GGF score for the policies learned by DQN, A2C, PPO and their GGF counterparts.
As expected, all the three GGF algorithms have higher GGF score than their original algorithms.
Higher GGF scores means the solution is more balanced which can be
validated from Figure~\ref{fig:species-bar10}.

\begin{figure}[H]
\vskip 0.2in
\begin{center}
\centering
\includegraphics[width=0.7\linewidth]{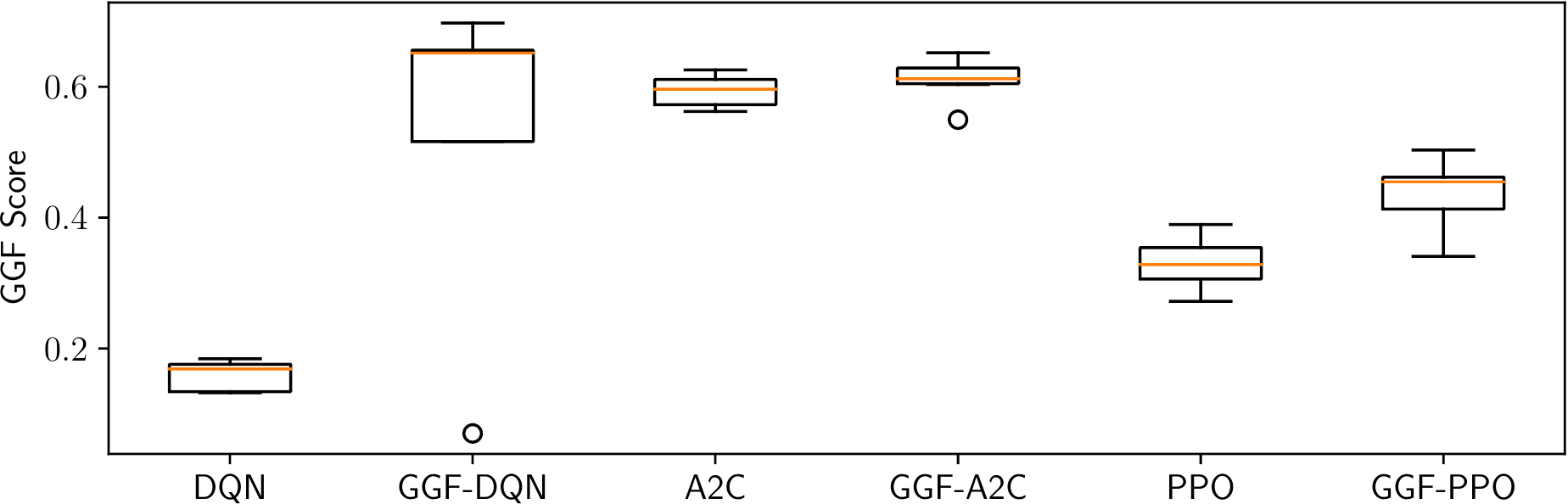}
\caption{GGF scores of DQN, A2C, PPO and their GGF algorithms with $\w_i = \frac{1}{10^i}$ during the testing phase in the SC domain.} 
\label{fig:species-box10}
\end{center}
\vskip -0.2in
\end{figure}

\begin{figure}[H]
\vskip 0.2in
\begin{center}
\centering
\includegraphics[width=0.7\linewidth]{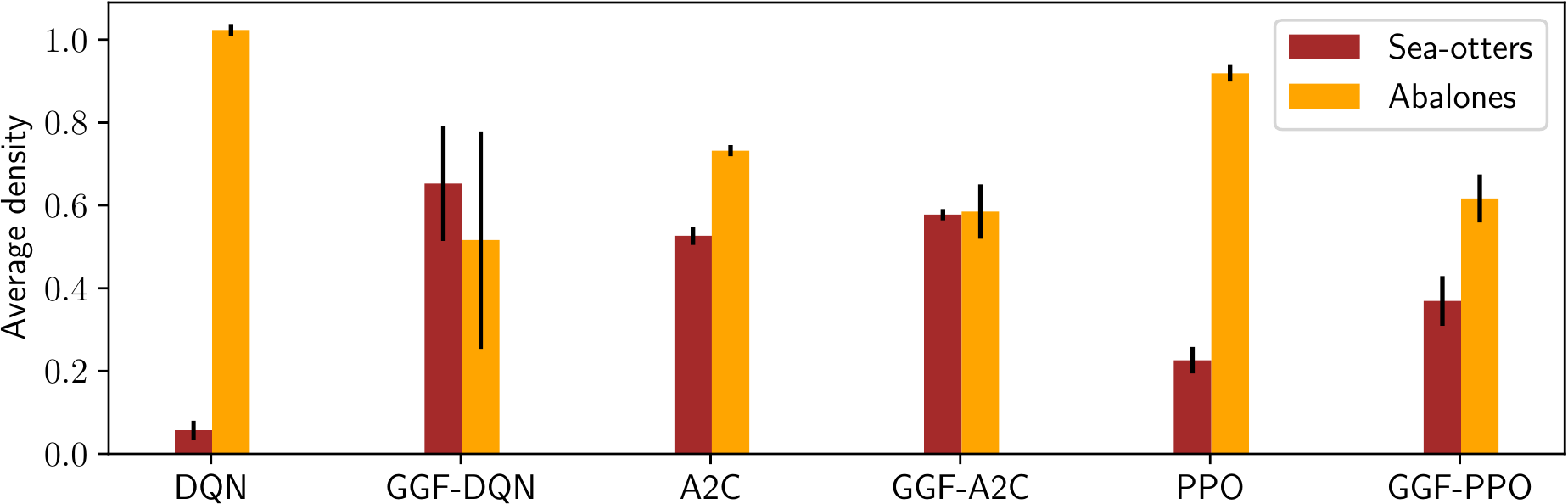}
\caption{Individual densities for DQN, A2C, PPO and their GGF versions with $\w_i = \frac{1}{10^i}$ during the testing phase in the SC domain.} 
\label{fig:species-bar10}
\end{center}
\vskip -0.2in
\end{figure}

Similar to the case of $\w_i = \frac{1}{2^i}$, we also compared those GGF algorithms with faster decreasing weights in terms of their CV, minimum and maximum of densities (Figure~\ref{fig:species-cv10}). 
As explained before, standard RL algorithms generate unequal distributions of rewards while our adapted versions of DQN, A2C and PPO generate more balanced solutions.
Again the CV of GGF algorithms is lower than their original algorithms which shows the less variations in their objectives. 

\begin{figure}[H]
\vskip 0.2in
\begin{center}
\centering
\includegraphics[width=0.7\linewidth]{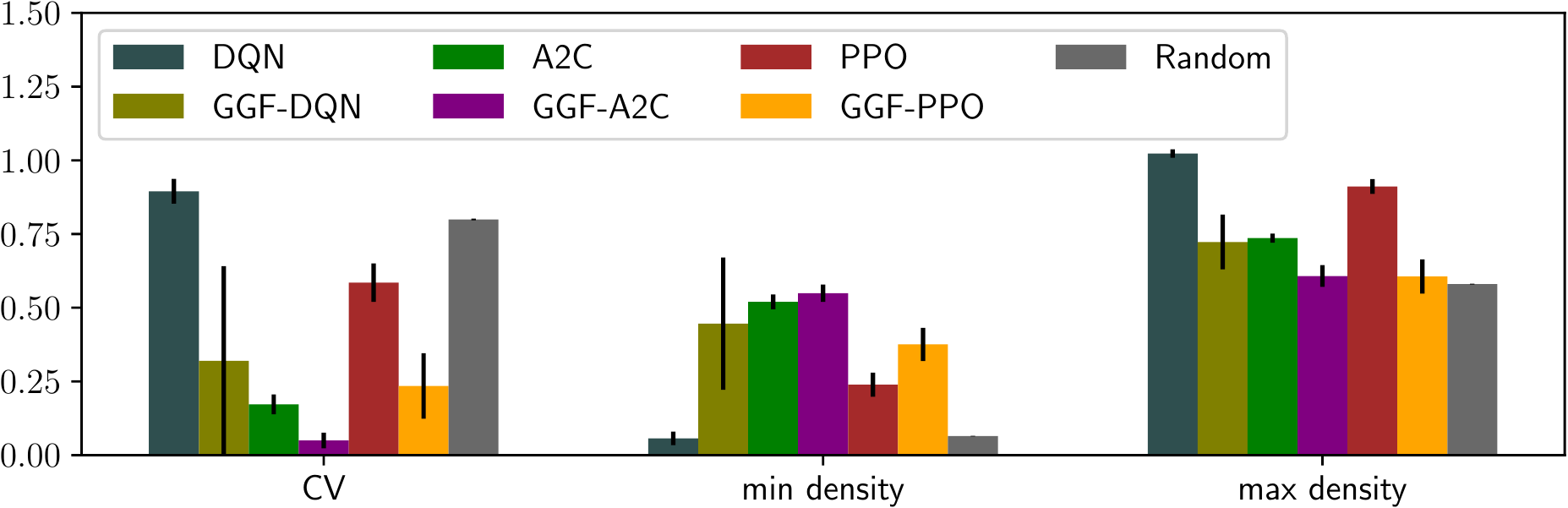}
\caption{The performances of different RL algorithms and their GGF versions with $\w_i = \frac{1}{10^i}$ in the SC domain.} 
\label{fig:species-cv10}
\end{center}
\vskip -0.2in
\end{figure}

\subsection{Traffic Light Control}

To clearly demonstrate that our proposition yields more equitable solutions, we comapre PPO, A2C, DQN and their GGF counterparts
in terms of waiting times per direction, which were estimated after training.
As shown in Figure~\ref{fig:ppo-bar}, the waiting times achieved by GGF-PPO is more balanced.

In terms of average waiting times, DQN and GGF-DQN did not work very well.
However, from the results in Figure~\ref{fig:dqn-bar}, it is clear that the GGF version of DQN is fairer than the standard DQN.

\begin{figure}[H]
\vskip 0.2in
\begin{center}
\centering
\includegraphics[width=0.5\linewidth]{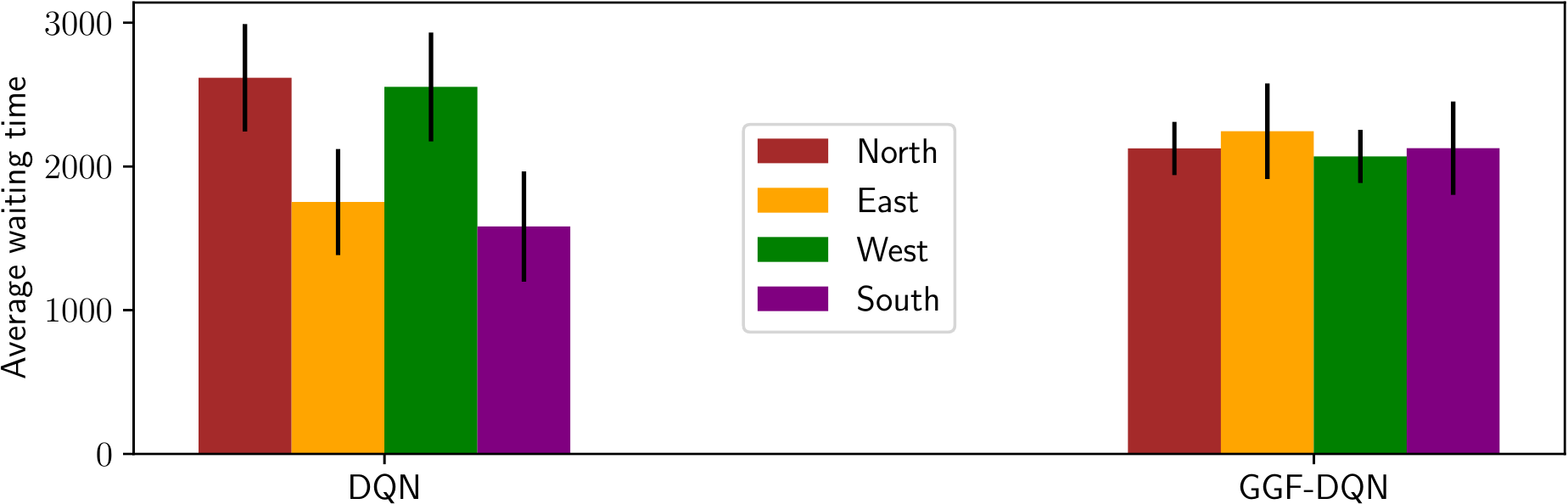}
\caption{Individual average waiting times of DQN and GGF-DQN during the testing phase.}\label{fig:dqn-bar}
\end{center}
\vskip -0.2in
\end{figure}

\begin{figure}[H]
\vskip 0.2in
\begin{center}
\centering
\includegraphics[width=0.5\linewidth]{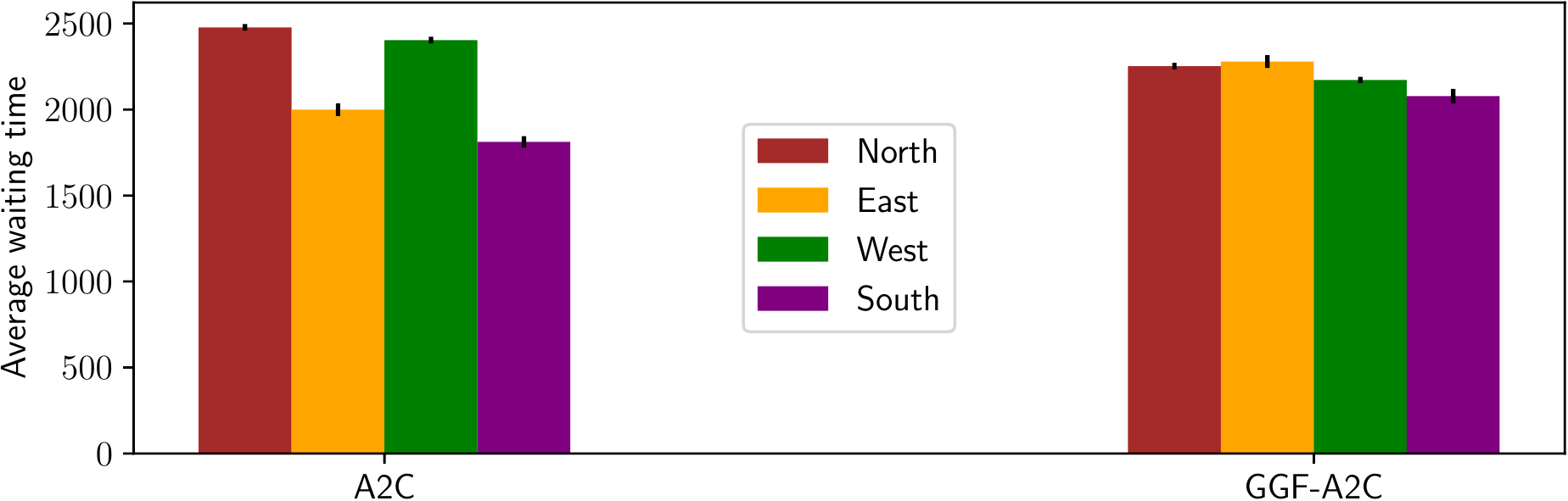}
\caption{Individual average waiting times of A2C and GGF-A2C during the testing phase.}\label{fig:a2c-bar}
\end{center}
\vskip -0.2in
\end{figure}

\begin{figure}[H]
\vskip 0.2in
\begin{center}
\centering
\includegraphics[width=0.5\linewidth]{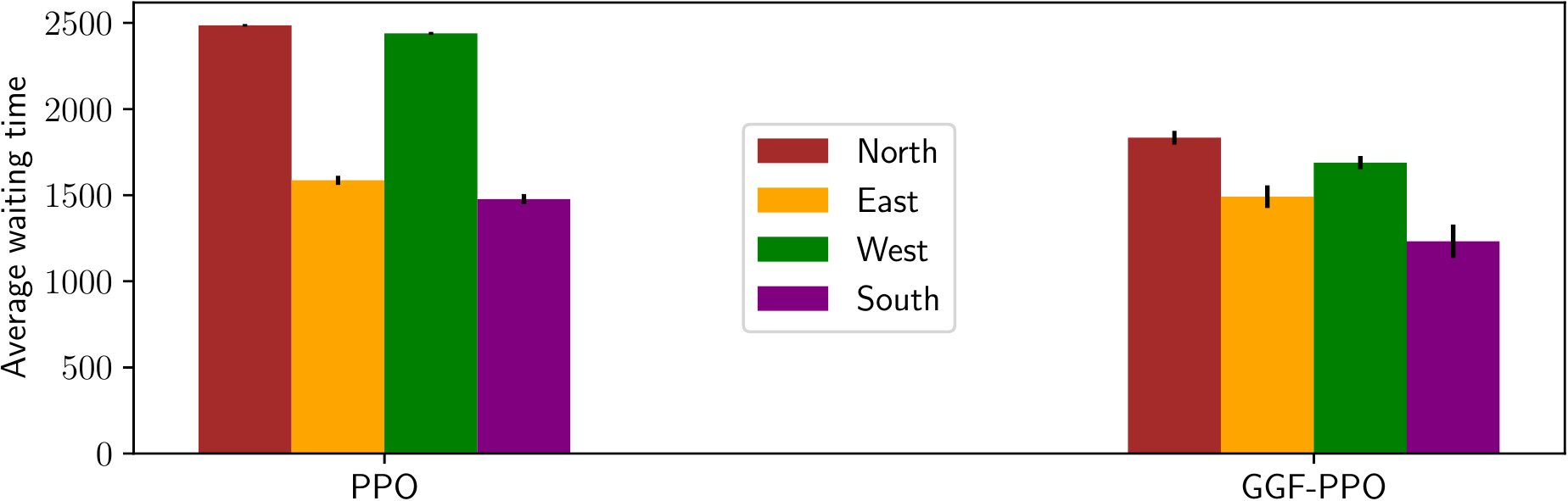}
\caption{Individual average waiting times of PPO and GGF-PPO during the testing phase.}\label{fig:ppo-bar}
\end{center}
\vskip -0.2in
\end{figure}

We also ran some additional experiments on a non-stationary environment.
For this experiment, the traffic generation is not fixed and changes during the day.
The problem becomes more challenging, but is much closer to a real environment.
There are many ways to add the variance in the traffic patterns. 
We defined 4 distributions (see~Figure~\ref{fig:dits almostistributions} for one lane) corresponding to four different periods of a day: morning, afternoon, evening, and night.

\begin{figure}[H]
\vskip 0.2in
\begin{center}
\centering
\includegraphics[width=0.6\linewidth]{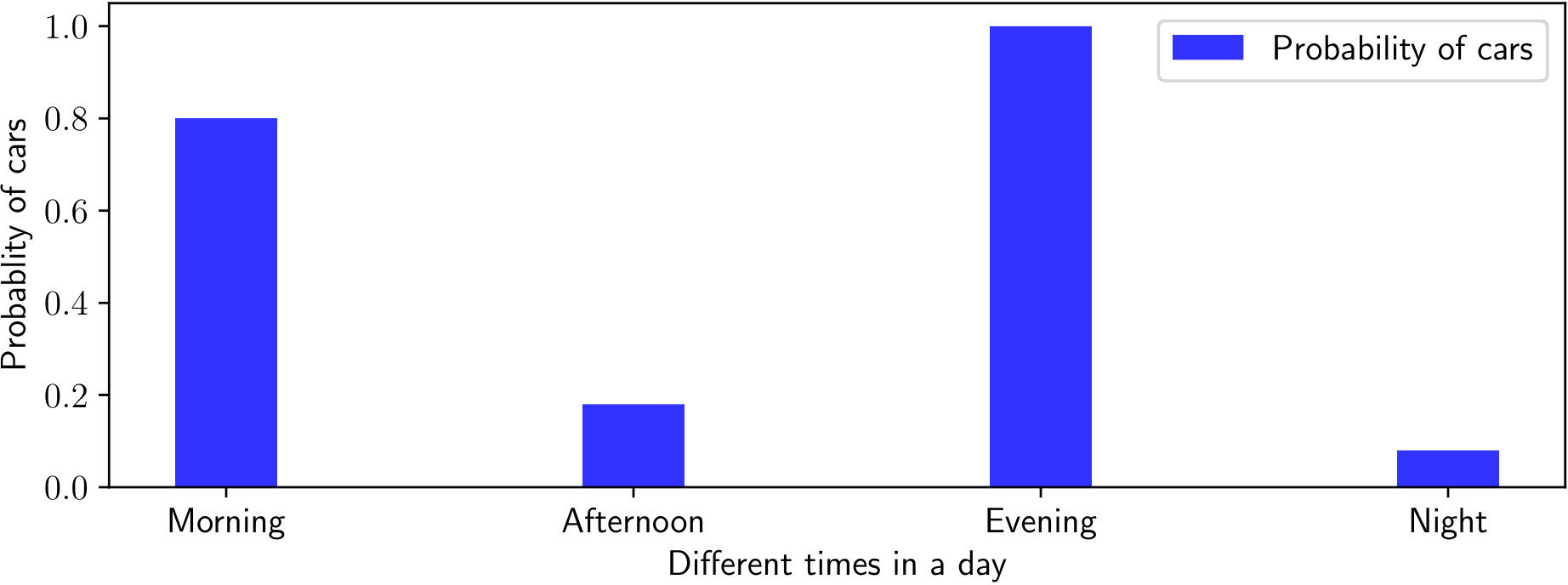}
\caption{Probabilities of car entering at different times in a day for one lane.} 
\label{fig:dits almostistributions}
\end{center}
\vskip -0.2in
\end{figure}

\begin{figure}[H]
\vskip 0.2in
\begin{center}
\centering
\includegraphics[width=0.6\linewidth]{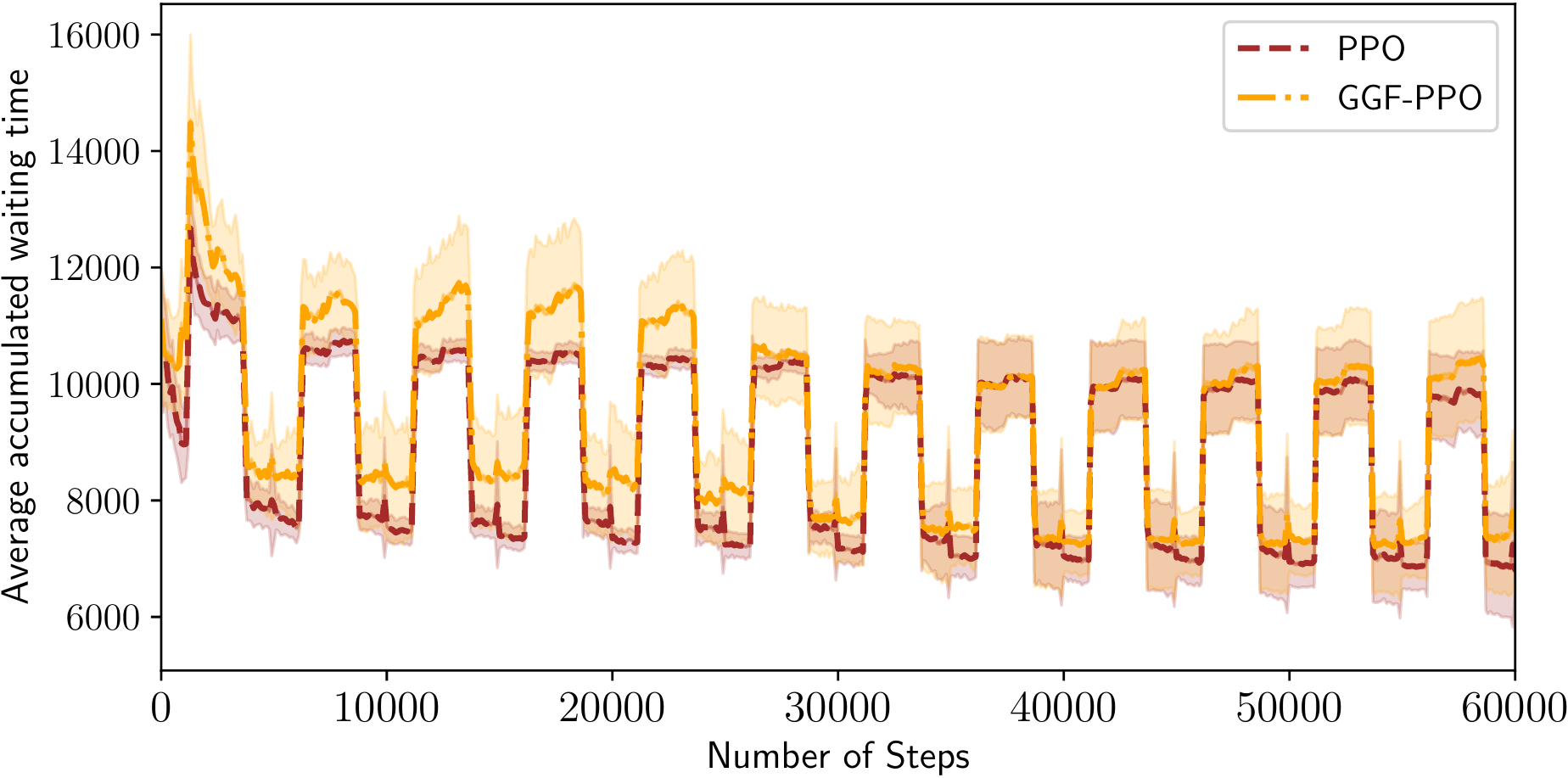}
\caption{Average waiting times of PPO and GGF-PPO during learning phase in the non-stationary TL domain.}
\label{fig:ppo-non-stat}
\end{center}
\vskip -0.2in
\end{figure}

Figure~\ref{fig:ppo-non-stat} visualizes the average accumulated waiting time of PPO and GGF-PPO on this non-stationary environment. 
As expected, GGF-PPO performs worse than PPO on that metric.  
The ups and downs represents the different times in a day. 
However, GGF-PPO achieves a much higher GGF score than PPO (see Figure~\ref{fig:ppo-non-stat-box}).

\begin{figure}[H]
\vskip 0.2in
\begin{center}
\centering
\includegraphics[width=0.6\linewidth]{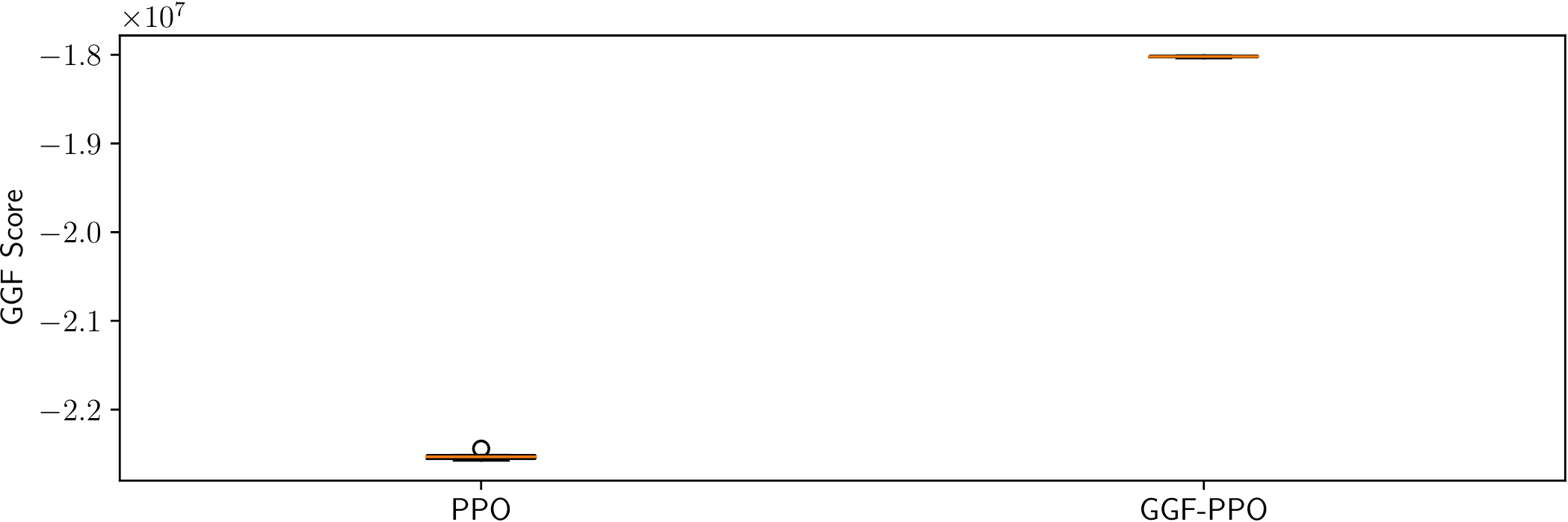}
\caption{GGF scores of PPO and GGF-PPO during the testing phase in the non-stationary TL domain.}\label{fig:ppo-non-stat-box}
\end{center}
\vskip -0.2in
\end{figure}

\subsection{Data Center Traffic Control}

\begin{figure}[H]
\vskip 0.2in
\begin{center}
\centering
\includegraphics[width=0.7\linewidth]{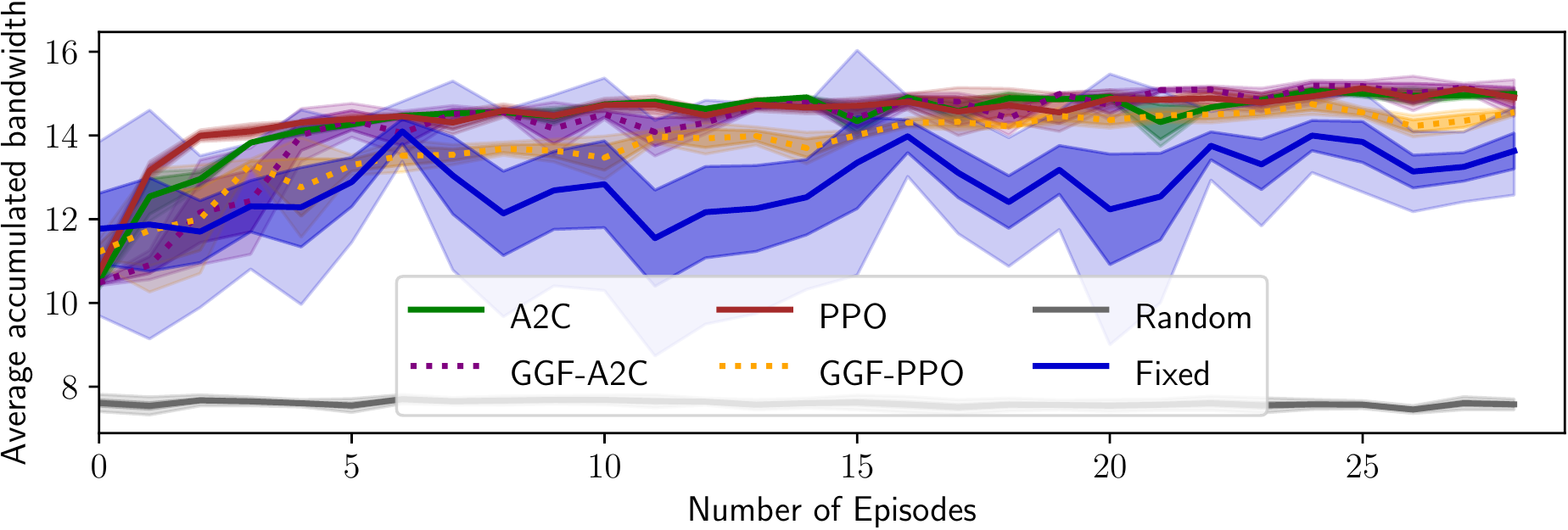}
\caption{Episodic reward of PPO and GGF-PPO during the learning phase, and that of the Fixed and Random policies in the DC domain over 20 runs with different seeds.}
\label{fig:iroko-episodic}
\end{center}
\vskip -0.2in
\end{figure}

Figure~\ref{fig:iroko-episodic} illustrates the learning curves of PPO, A2C and their GGF counterparts in terms of episodic rewards.
The performance of two policies (a fixed one and a random one) are added for comparison.
The fixed policy always chooses the maximum bandwidth for each host.
The random policy selects actions with a uniform distribution.
We can see that PPO and GGF-PPO converge to a much higher reward than random and fixed policies. 
GGF-PPO's and GGF-A2C's average bandwidth in an episode is lower than PPO and A2C, this is because it is hard for a single policy to maximize rewards while ensuring fairness. 
However, GGF algorithms still performs better than the random and fixed policies and tries to get high rewards while allocating bandwidth equally to different hosts. 

\end{document}